\documentclass[conference,10pt]{EHAI/IEEEtran}
\IEEEoverridecommandlockouts
\usepackage{cite}
\usepackage{cuted} 
\usepackage{amsmath,amssymb,amsfonts}
\usepackage{amsthm}
\usepackage{booktabs}
\theoremstyle{definition}
\usepackage{algorithmic}
\usepackage{graphicx}
\usepackage{textcomp}
\usepackage{xcolor}
\usepackage{romannum}
\usepackage{multirow}
\usepackage{etoolbox}
\usepackage{url}
\makeatletter
\newcommand{\autobreak}[2][0pt]{%
  \ifbool{@twocolumn}{\\[#1] #2{} &}{#2}%
}
\newcommand{\autobreakindent}{\ifbool{@twocolumn}{\\ & \qquad \qquad \quad}{}}
\makeatother
\usepackage{subfigure}
\usepackage{cleveref}
\usepackage{caption,setspace}
\usepackage{adjustbox}
\usepackage{caption}
\usepackage[font=footnotesize]{caption}
\usepackage{amsmath,amssymb,amsthm,mathtools}
\usepackage{stfloats}

\theoremstyle{plain}
\newtheorem{theorem}{Theorem}

\newtheorem{remark}{Remark}

\newtheorem{proposition}{Proposition}
\newtheorem{corollary}{Corollary}
\usepackage{graphicx}
\usepackage{algorithm,algorithmic}
\newcommand{\Pcomp}{P_{\mathrm{comp}}}

\def\BibTeX{{\rm B\kern-.05em{\sc i\kern-.025em b}\kern-.08em
    T\kern-.1667em\lower.7ex\hbox{E}\kern-.125emX}}

\begin{document}

\pagenumbering{arabic}
\setcounter{page}{1}

\title{Space Generative AI with Solar Energy Harvesting\\

\thanks{This work has been submitted to the IEEE for possible publication.  Copyright may be transferred without notice, after which this version may no longer be accessible.}

\thanks{
J. Zhang, J. Huang, Z. Wang, and K. Huang are with the Department of Electrical
and Computer Engineering, The University of Hong Kong (HKU), Hong
Kong SAR, China. Emails: jrzhang@eee.hku.hk, jianhaoh@hku.hk, zhanweiw@eee.hku.hk, and huangkb@eee.hku.hk (Corresponding author: K. Huang).}
}

\author{\IEEEauthorblockN{Jierui Zhang, Jianhao Huang, Zhanwei Wang, and Kaibin Huang
}}

\maketitle

\begin{abstract}

Satellites are emerging as promising platforms for extending generative \emph{artificial intelligence} (AI) services to remote areas lacking terrestrial infrastructure. However, deploying space generative AI is fundamentally constrained by the limited and time-varying onboard energy supplied by solar \emph{energy harvesting} (EH). This paper presents a framework for solar-powered space generative AI in which a satellite receives a user prompt, executes a diffusion-based image-generation model, and downlinks the compressed result within a strict time window. To design this framework, we identify fundamental \emph{computation--communication} (C$^2$) trade-offs governed by the shared harvested-energy budgets. Specifically, increasing the number of generation steps improves intrinsic image quality but depletes energy and time available for downlink transmission, whereas prioritizing communication guarantees reliable delivery but sacrifices semantic quality. To balance these trade-offs and maximize \emph{end-to-end} (E2E) generative performance, we exploit the predictable solar-EH dynamics induced by deterministic orbital motion and develop a joint C$^2$ resource-optimization framework using a tractable two-step approach. First, we characterize the maximum downlink throughput for a fixed generation depth under continuous solar EH. This establishes a separation principle that decouples waiting-time selection from optimal transmit-power control. Next, building on this principle, we formulate a joint C$^2$ utility-maximization problem and derive a closed-form, low-complexity step-selection policy in the dominant constant-power regime. Extensive experiments under realistic orbital dynamics demonstrate that the proposed policy dynamically balances generation quality and transmission reliability. This yields significant E2E performance gains over static computation- and communication-centric baselines across diverse solar EH states.

\end{abstract}


\section{Introduction}

With rapid advances in generative \emph{artificial intelligence} (AI) such as high-fidelity image generation and complex reasoning, supporting globally accessible generative-AI services is becoming a critical objective for future networks~\cite{ho2020denoising,wei2022chain,qu2025mobile,wang2026revisiting}.
However, delivering generative-AI services to remote and underserved areas remains challenging due to the lack of terrestrial AI infrastructure. Although satellites have traditionally served as passive relays connecting remote users to ground networks, this relay-based paradigm may incur significant latency. Enabled by increasing onboard computation capability, satellites are emerging as active AI-service providers for ubiquitous generative-AI service provisioning beyond the coverage of terrestrial networks~\cite{chen2024survey,chen2024fedmeld,esmat2023toward,chen2025space,ji2024dynamic}.  
Meanwhile, satellites can harvest solar energy directly in orbit, avoiding atmospheric attenuation and many ground-level obstructions, thereby supporting onboard tasks through solar \emph{energy harvesting} (EH).
These advantages have motivated two related architectural visions: space data centers, which deploy large-scale computing infrastructure in orbit, and space-ground integrated computing power networks, which coordinate computing resources across satellites, terrestrial edge nodes, and cloud data centers~\cite{sun2026toward,xiao2026sgicpnom,wang2026spacemoe}. A recent example is StarCloud, which launched a prototype orbital computing facility equipped with an H100 GPU for onboard model training and inference, further indicating the feasibility of executing AI workloads in orbit~\cite{starcloud_2025_h100}.
However, realizing practical space generative AI still requires overcoming a substantial challenge: onboard AI inference and downlink transmission typically rely on time-varying energy supplied by solar EH, while computation and communication are tightly coupled through the shared onboard energy budget. 
Although sun-synchronous orbits can mitigate solar-energy variation in some cases, such favorable orbits cannot accommodate all space-AI deployments due to limited orbital resources and mission-specific coverage requirements. Hence, space AI systems must also be designed for general orbital conditions, where solar EH can be time-varying.
This paper investigates the joint management of \emph{computation and communication} (C\(^2\)) for space generative AI, with the goal of maximizing \emph{end-to-end} (E2E) performance under solar-EH constraints.

The deployment of generative AI is inherently energy-intensive due to its heavy computational workload, posing significant challenges for energy-constrained onboard AI execution~\cite{samsi2023words}. Unlike terrestrial AI systems that typically have access to stable power-grid infrastructure, orbital AI must rely on limited onboard energy resources.
To sustain off-grid operations, EH has emerged as a key enabling technology and has attracted extensive research attention, particularly in wireless communications~\cite{ku2015advances}. In satellite systems, solar EH is the primary energy source, but the harvested power is not a static budget. Instead, it varies continuously with orbital motion, solar incidence angle, panel orientation, and sunlight-eclipse transitions~\cite{yang2016towards,liu2024orbit,li2024battery}. 
Although the satellite-ground visibility window may be much longer than a single service request, an interactive generative-AI task typically imposes a much shorter execution window to ensure a timely response~\cite{qu2026trimcaching}. This short window limits both the time available for energy accumulation and the duration available for onboard computation and downlink transmission. As a result, the satellite cannot simply wait until sufficient solar energy is harvested, but must complete waiting, computation, and communication within a tight service deadline. During this window, the available onboard energy is strongly determined by the satellite's orbital position and solar-EH-induced battery state. 
Therefore, completing such a task requires careful allocation of limited harvested energy and time between onboard generation and downlink transmission.

This allocation problem is particularly challenging for space generative AI because computation and communication jointly determine the final received content quality.  
Consider a text-to-image service where a ground device uploads a lightweight prompt to a satellite, the satellite generates an image onboard using a diffusion model~\cite{rombach2022high,song2020denoising}, and the generated image is then compressed and downlinked to the ground device. In this pipeline, computation and communication draw from the same harvested-energy buffer. Increasing the number of denoising steps generally improves the intrinsic quality of the generated image, but it consumes more energy and leaves less energy for downlink transmission. Conversely, reserving more energy for transmission increases the delivery capacity and can reduce compression distortion, but it forces the generative process to use fewer denoising steps, which degrades semantic quality. As a result, a computation-centric policy synthesizes a high-quality image that cannot be faithfully delivered, while a communication-centric policy reliably delivers an image whose intrinsic quality is poor. Therefore, maximizing the E2E generation quality requires task-aware allocation of energy between onboard generation and downlink transmission under solar-EH-induced causality constraints.

Existing studies provide useful foundations, but they do not fully address the above generate-and-deliver trade-off. A first line of work is EH communications, which studies transmission scheduling under energy-causality constraints. Classical results characterize optimal throughput-maximizing policies through methods such as directional water filling and save-then-transmit strategies~\cite{ozel2011transmission,tutuncuoglu2012optimum,yang2011optimal,luo2013optimal}. 
These works establish important principles for energy-causal transmission, but they typically allocate harvested energy solely to communication, without considering computation that may compete for the same energy buffer.
A second related line considers energy-aware satellite systems. Several studies have investigated satellite energy management, sunlight-aware task scheduling, and battery-aware edge computing under orbital energy constraints~\cite{yang2016towards,liu2024orbit,li2024battery}. These works highlight the importance of orbital energy dynamics, but they mainly focus on routing, task scheduling, or generic computing workloads. They do not model the task-specific quality behavior of generative AI, where the computation depth directly affects the semantic quality of the output.
Finally, recent studies have explored joint computation and communication
resource allocation, as well as broader sensing--communication--computation
integration\cite{wen2026ISCC}. Some works minimize download time or information freshness by jointly optimizing onboard processing and transmission~\cite{ouyang2023joint,li2025age,cao2023computing}, while others focus on energy efficiency, computation offloading, compression, or network throughput~\cite{ding2021joint,wang2023energy,dai2025energy,qian2025energy,he2025joint,lai2024joint,gong2022computation,zhang2022exploiting}.
However, most of these formulations treat computation as a fixed workload or a generic offloading/compression task. They do not capture a key property of diffusion-based generation: the number of denoising steps is a controllable quality knob that should be jointly adapted with downlink transmission. Moreover, solar EH is rarely considered in generative-service-oriented satellite designs. Consequently, existing EH communication and satellite C\(^2\) frameworks are insufficient for optimizing E2E generative performance under solar EH.

The key challenges of optimizing such a framework arise from two factors: (1) the time-varying profiles of both the wireless channel and solar EH, and (2) the coupling between onboard generation and subsequent downlink transmission under the solar-EH profile. For the first challenge, deterministic orbital motion and geometric relationships make these profiles predictable over the task execution window and provide tractable mathematical models, thereby enabling proactive resource allocation. For the second challenge, onboard generation and downlink transmission jointly determine the E2E performance, requiring the system to manage generation quality, waiting time, and communication throughput under a shared solar-EH-powered energy buffer, rather than merely scheduling transmit power as in conventional EH communication. To study this joint C\(^2\) optimization for space generative AI, we decompose the E2E performance into two coupled components: onboard generation quality and downlink throughput.
We first fix the generation step and optimize the waiting time and transmit power to characterize the maximum achievable downlink throughput under continuous energy causality.
Building on this characterization, we then formulate a joint C\(^2\) utility that balances onboard generation quality and downlink throughput, and derive a closed-form, low-complexity step-selection rule in the dominant constant-power regime.
Specifically, the key contributions and findings are summarized as follows:
\begin{itemize}
    \item \textbf{Communication Throughput Maximization under Solar EH:}
     We fix the generation step and investigate a communication-throughput maximization problem. This problem not only serves as the inner subproblem of the subsequent joint C\(^2\) optimization, but also becomes critical when the generated content has a large downlink data volume. Unlike classical EH communication, which provides general structural solutions for arbitrary EH profiles, our problem involves an orbital solar-EH profile and a waiting-time decision introduced by onboard task execution. To solve the problem, we instantiate a tractable, representative solar-EH model based on orbital geometry, prove a separation principle between waiting-time selection and transmit-power control, and derive the optimal energy-feasible transmit-power policy using the \emph{lower convex envelope} (LCE) principle.
    \item \textbf{Joint C$^2$ Optimization under Solar EH:} 
     Beyond the fixed-step communication-throughput optimization, we further consider the general case where onboard generation quality and downlink throughput must be jointly balanced.
    Specifically, we approximate the relationship between onboard generation quality (measured by \emph{CLIP score} \cite{radford2021learning,cai2024condition}) and generation steps, and then formulate a joint C$^2$ utility maximization problem over the generation step, waiting time, and transmit-power policy. This problem is challenging because these control variables are tightly coupled through the shared solar-EH-powered energy buffer. By decomposing the problem with respect to the generation step, the inner subproblem is connected to the throughput-maximization problem derived above. In the dominant constant-power regime, this decomposition yields a reduced single-variable problem, for which we derive a closed-form, low-complexity step-selection rule using the Lambert \(W\) function. The resulting rule enables lightweight onboard adaptation and reveals how the optimal generation step changes with the solar-EH state and system parameters.
    \item \textbf{Experimental Results:} We conduct comprehensive experiments to evaluate the proposed space generative AI framework under realistic \emph{low-Earth-orbit} (LEO) parameters. The results show that the proposed joint optimization scheme adapts between computation-centric and communication-centric schemes according to the solar-EH state. It approaches the more suitable extreme policy in energy-scarce or energy-abundant regimes, while outperforming both extremes in moderate-energy regimes through balanced onboard generation and downlink transmission. Overall, it achieves more robust E2E generation performance than fixed extreme allocation schemes.
\end{itemize}

The remainder of this paper is organized as follows. Section~\ref{system model} introduces the models and metrics. Section~\ref{Communication Throughput Maximization} solves the communication throughput maximization problem. Section~\ref{Utility Maximization} investigates a joint C$^2$ optimization problem and derives the closed-form solution. Experimental results are provided in Section~\ref{Experimental Results}, followed by concluding remarks in Section~\ref{Concluding Remarks}.

\section{System Model}
\label{system model}

Consider a space generative AI framework powered by solar EH, as shown in Fig.~\ref{fig:system}. 
In this framework, a static ground device offloads prompt-based generation tasks to a satellite. Then, the solar-EH-powered satellite completes a sequence of onboard operations within a short task execution window of duration \(T\), including waiting, content generation, and result transmission. The solar EH model, onboard computation model, space-ground communication model, and performance metrics are detailed in the following subsections.

\begin{figure*}[t]
\centering
\subfigure[Space generative AI framework.]{    \includegraphics[width=0.48\textwidth]{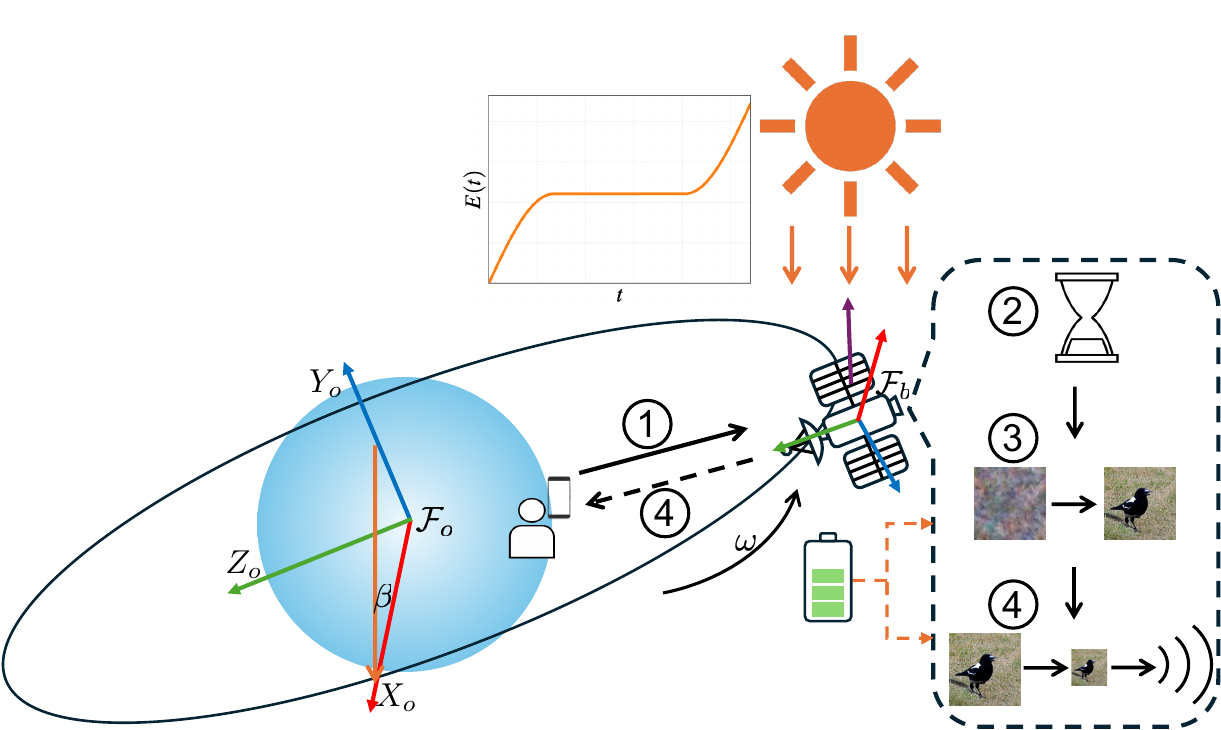}
    \label{fig:system_model}
}
\hspace{0.05\textwidth}
\subfigure[Operations and protocol.]{   \includegraphics[width=0.23\textwidth]{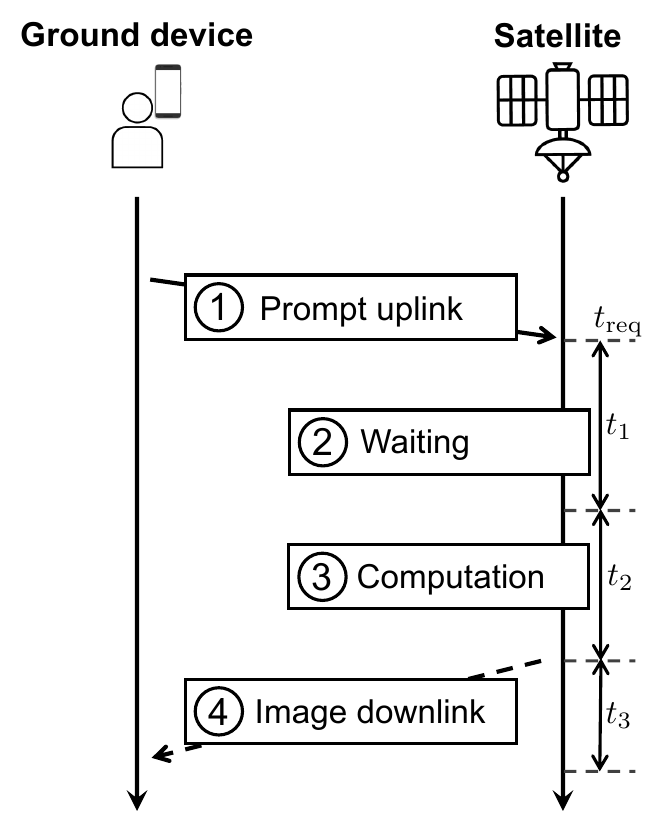}   \label{fig:protocol}
}
\caption{System model of space generative AI.}
\label{fig:system}
\vspace{-4ex}
\end{figure*}

\subsection{Solar EH Model}

In this subsection, we model satellite solar EH based on orbital geometry. Two Cartesian coordinate systems, the \emph{orbital reference frame} ($\mathcal{F}_o$) and the \emph{satellite body frame} ($\mathcal{F}_b$)~\cite{wertz2012spacecraft}, are as shown in Fig.~\ref{fig:system_model}.
\begin{itemize}
    \item \textbf{Orbital Reference Frame ($\mathcal{F}_o$):} The origin is located at the center of the Earth. The $Y_o$-axis is orthogonal to the orbital plane. The $X_o$-axis points towards the projection of the sun vector onto the orbital plane, and the $Z_o$-axis completes the right-handed system. The sun vector is
    \begin{equation}
        \mathbf{S}_o = [\cos\beta, \sin\beta, 0]^T,
    \end{equation}
    where \(\beta\) denotes the solar elevation angle.
    \item \textbf{Satellite Body Frame ($\mathcal{F}_b$):} The origin is the satellite's center of mass. We adopt a standard nadir-pointing attitude \cite{cheriet2021inertia}, where $z_b$ points towards the Earth (nadir), $y_b$ is aligned with the negative orbit normal (cross-track) and $x_b$ is aligned with the velocity vector (along-track).
\end{itemize}
The angular velocity $\omega$ of the satellite is determined by Kepler's third law, i.e., 
$\omega = \sqrt{\frac{GM}{(R_{\mathrm{E}} + H)^3}},$
where $G$ is the gravitational constant, $M$ is the mass of the Earth, $R_{\mathrm{E}}$ is the Earth radius, and $H$ is the orbital altitude\cite{bate2020fundamentals}. Consider a solar panel rigidly mounted on the satellite body, with unit normal vector \(\mathbf{n}= [n_x,n_y,n_z]^T \in \mathbb{R}^3\) expressed in \(\mathcal{F}_b\).
Assume that $t=0$ corresponds to the orbital noon and that $\mathbf{n} = [0,0,-1]^T$ (i.e., the anti-nadir direction); then the solar-EH profile (the instantaneous harvested power) is\footnote{The orbital period of LEO satellite (tens of minutes) is negligible relative to that of Earth's revolution around the Sun. Hence, we assume that the sunlight direction and the orbital plane remain constant within a single orbital period.}
\begin{equation}
    P(t)=
\begin{cases}
P_0\cos(\omega t), & 0 \le t \le \dfrac{\pi}{2\omega},\\[4pt]
0, & \dfrac{\pi}{2\omega} \le t \le \dfrac{3\pi}{2\omega},\\[4pt]
P_0\cos(\omega t), & \dfrac{3\pi}{2\omega} \le t \le \dfrac{2\pi}{\omega}.\label{EH power}
\end{cases}
\end{equation}
Here, $P_0=\eta \cdot \gamma \cdot A \cdot \cos\beta$, where $\eta$ is the power conversion efficiency, $\gamma$ is the solar irradiance per unit area, and $A$ is the area of the solar panel~\cite{yang2016towards}. Integrating \eqref{EH power} from the orbital-noon reference yields the following representative cumulative-energy profile:
\begin{equation}
    E(t) =
\begin{cases}
\dfrac{P_0}{\omega}\sin(\omega t)+E_0, & 0 \le t \le \dfrac{\pi}{2\omega},\\[4pt]
\dfrac{P_0}{\omega}+E_0, & \dfrac{\pi}{2\omega} \le t \le \dfrac{3\pi}{2\omega},\\[4pt]
\dfrac{P_0}{\omega}\bigl(2 + \sin(\omega t)\bigr)+E_0, & \dfrac{3\pi}{2\omega} \le t \le \dfrac{2\pi}{\omega},\label{EH energy}
\end{cases}
\end{equation}
where $E_0 \ge 0$ denotes the initial energy at task arrival.

\subsection{Onboard Generative AI Model}

 Although generative AI supports various modalities, text-to-image generation naturally fits the considered scenario: a power-limited ground user uploads only a lightweight prompt, while the solar-EH-powered satellite generates and downlinks the larger image result. Hence, we consider a typical text-to-image generation task, as described below.

\subsubsection{Generative AI Model}

We adopt a representative \emph{latent diffusion model} (LDM) \cite{rombach2022high} as the generative backbone for text-to-image generation. The inference pipeline consists of three distinct phases as follows. First, a pre-trained text encoder $\tau_\theta$ processes the input prompt to produce conditional embeddings $\mathbf{c}$. Second, initialized with random Gaussian noise $\mathbf{z}_n$ in the latent space, the diffusion backbone iteratively refines the latent representation over $n$ steps using the \emph{denoising diffusion implicit model} (DDIM) scheduler conditioned on $\mathbf{c}$~\cite{song2020denoising}. Finally, the VAE decoder $\mathcal{D}$ projects the clean latent $\mathbf{z}_0$ back into the pixel space to generate the final image $\mathbf{x}$. 

\subsubsection{Computation Energy Consumption}

We assume that, for each task, the computation power \(P_{\mathrm{comp}}\) is fixed. Let $\Omega_1$ and $\Omega_2$ denote the \emph{number of floating point operations} (FLOPs) of each DDIM step and the FLOPs of remaining operations for a single inference (e.g., decoder of the VAE). Let $\nu$ denote the computation speed. Then, the computation time is
\begin{equation}
    t_2 = c_1 n + c_2,\label{t2}
\end{equation}
where $c_1=\frac{\Omega_1}{\nu}$ and $c_2=\frac{\Omega_2}{\nu}$.
Since the satellite may not possess sufficient energy to execute the task immediately upon reception, a waiting phase $t_1$ may be required. Hence, the cumulative computation energy consumption is given by
\begin{equation}
\begin{aligned}
E_{\mathrm{comp}}(t)\!=\!
\begin{cases}
0, &\!\!\!\! t_{\mathrm{req}} \le t < t_{\mathrm{req}} \!+\! t_1,\\[-0.5ex]
P_{\mathrm{comp}} (t - t_{\mathrm{req}} - t_1), &\!\!\!\! t_{\mathrm{req}}\! +\! t_1 \le t < a,\\[-0.5ex]
P_{\mathrm{comp}}\, t_2, &\!\!\!\! a \le t \le b,
\end{cases}
\end{aligned}
\end{equation}
where \(a \triangleq t_{\mathrm{req}}+t_1+t_2\) and \(b \triangleq t_{\mathrm{req}}+T\).

\subsection{Space-ground Communication Model}

We focus on downlink transmission because the prompt-uplink traffic is negligible.
Consider a static ground device located on the Earth's surface.
We define the \emph{delivery capacity} as the achievable number of bits delivered over the communication window \([a,b]\).
In this work, the satellite can exploit non-causal information because both the orbital trajectory and the space-ground channel evolution are predictable. Hence, the delivery capacity and the corresponding transmit power policy can be determined in advance (detailed in Section~\ref{Communication Throughput Maximization}). Then, the generated image \(\mathbf{x}\) is compressed into \(\mathbf{x}'\) such that its data size does not exceed the delivery capacity. The compressed image is then transmitted according to the optimized power policy.
The instantaneous channel gain $g(t)$ is
\begin{equation}
    g(t) = G_{\mathrm{tx}} G_{\mathrm{rx}}L_{\mathrm{FSPL}}L_{\mathrm{AL}},
    \label{FSPL}
\end{equation}
where $G_{\mathrm{tx}}$ and $G_{\mathrm{rx}}$ denote the transmit and receive antenna gains, respectively~\cite{chen2024fedmeld,cai20253c}. $L_{\mathrm{FSPL}}=\big(\frac{\ell}{4\pi d(t)}\big)^2$ represents the large-scale free-space path-loss factor, where $\ell = \frac{c}{f_c}$ is the downlink carrier wavelength, $c$ is the speed of light, $f_c$ is the carrier frequency, and $d(t)$ is the time-varying distance between the satellite and the ground device. Note that $d(t)$ is predictable based on the satellite's orbital dynamics, thereby leading to a predictable channel. $L_{\mathrm{AL}}$ accounts for additional environmental attenuation, such as rain fading. Small-scale fading is omitted, due to the nature of the considered remote area~\cite{chen2024fedmeld}. According to Shannon's theorem \cite{shannon1948mathematical}, the instantaneous communication rate is given by
\begin{equation}
    r(t) = B \log_2 \left( 1 + \frac{P_{\mathrm{comm}}(t) g(t)}{N_0 B} \right),
    \label{r(t)}
\end{equation}
where $P_{\mathrm{comm}}(t)$ is the transmit power, $N_0$ is the noise power spectral density, and $B$ is the bandwidth. 
Then, the cumulative communication energy consumption is given by 
\begin{equation}
\begin{aligned}
E_{\mathrm{comm}}(t)=
\begin{cases}
0, & t_{\mathrm{req}} \le t < a,\\[-0.5ex]
\int_{a}^{t} P_{\mathrm{comm}}(\tau)\, \mathrm{d}\tau, & a \le t \le b.
\end{cases}
\end{aligned}
\end{equation}

\subsection{Performance Metrics}
\label{metrics}

We aim to maximize the E2E generation quality, measured by the E2E \emph{CLIP score}, i.e., the semantic similarity between the text prompt and the image received at the ground device~\cite{radford2021learning,cai2024condition}. However, directly characterizing this E2E metric in closed form is challenging, as it depends jointly on the generative model, compression process, and downlink transmission.
Therefore, we instead consider two tractable factors, corresponding to computation and communication, respectively, that jointly determine the E2E metric:
\begin{itemize}
    \item \textbf{Onboard CLIP score:} This metric measures the semantic similarity between the text prompt and the onboard-generated image, reflecting its intrinsic quality. Empirical results show that the CLIP score \(S\) increases with the number of DDIM steps \(n\).
    \item \textbf{Communication throughput:} This metric is defined as the total number of bits delivered over the communication window~\cite{ozel2011transmission}, given by
    \begin{equation}
        R = \int_{a}^{b} r(t) \, \mathrm{d}t. 
        \label{throughput}
    \end{equation}
    A higher communication throughput allows a lower compression ratio and thus better preserves the quality of the generated image. It is related to waiting time, computation time, and transmit power policy.
\end{itemize}

\begin{remark}
    Under solar-EH constraints, these two factors exhibit an inherent C\(^2\) trade-off: allocating more energy to onboard generation improves intrinsic image quality, but reduces the energy available for downlink transmission, and vice versa. This motivates the analytical design in the following two sections: Section~\ref{Communication Throughput Maximization} studies communication-throughput maximization when fixing computation, while Section~\ref{Utility Maximization} jointly optimizes C\(^2\) for E2E performance maximization.
\end{remark}

\section{Communication Throughput Maximization under Solar EH}
\label{Communication Throughput Maximization}

In this section, we maximize communication throughput under solar EH, assuming a fixed computation time. 
This formulation constitutes a critical subproblem of the joint optimization framework presented in Section \ref{joint C2}.

\subsection{Problem Formulation}

 The communication throughput maximization problem is formulated in \eqref{P1}, where we fix $n$ and thus $t_2$ is also fixed.

\begin{subequations}
\label{P1}
\begin{align}
\mathrm{(P1)} 
\max_{t_1,\,P_{\mathrm{comm}}(\cdot)}\ \
& \int_a^b r(\tau)\,\mathrm{d}\tau
\label{opt goal P1}
\\[-0.2ex]
\text{s.t.}\ \
& P_{\mathrm{comp}}(t\!-\!t_{\mathrm{req}}\!-\!t_1) \!\le\! E(t),
\ \forall t\!\in\![t_{\mathrm{req}}\!+\!t_1,\!a],
\label{comp constraint P1}
\\[-0.2ex]
& \int_a^t\!\! P_{\mathrm{comm}}\!(\tau)\mathrm{d}\tau
\!\le\! E(t)\!-\!P_{\mathrm{comp}}t_2,\ \forall t\!\in\![a,\!b],
\label{comm constraint P1}
\end{align}
\end{subequations}
where $r(t)$ is given in \eqref{r(t)}.
Since the waiting time \(t_1\) is continuous and the transmit power \(P_{\mathrm{comm}}(\cdot)\) is a continuous-time function, directly solving problem (P1) is highly complex. A possible approach is to discretize the feasible range of \(t_1\) and, for each discretized value, apply the continuous-EH transmission result in~\cite{varan2014energy}. However, this discretization generally yields only an approximate solution and incurs high computational complexity when a fine time resolution is required. Therefore, we first investigate the structural properties of the optimal policy in the following subsection.

\subsection{Properties of Optimal Solution}\label{solution structure}

To derive the exact solution of (P1), we first analyze the properties of the optimal solution. The property of the optimal waiting time is characterized in the following proposition.

\begin{proposition}[Optimal Waiting Time]\label{prop t1}
    To achieve the optimal solution, the waiting time $t_1$ should be minimized, provided that the energy causality constraint for computation \eqref{comp constraint P1} is satisfied. Specifically, the optimal value $t_1^*$ is given by
\begin{equation}
\begin{split}
t_1^*
= &\inf \Big\{\, t_1 \in [0,T-t_2] \;\Big|\;
\autobreak[-1ex]
P_{\mathrm{comp}}(t-t_{\mathrm{req}}-t_1) \le E(t), \; \forall\, t \in [t_{\mathrm{req}}+t_1,a]
\,\Big\}.
\end{split}
\label{eq:t_opt}
\end{equation}
\end{proposition}

\begin{proof}
    Let $t_1^*$ denote the minimal feasible waiting time defined in \eqref{eq:t_opt}. We show that any $\hat{t}_1 > t_1^*$ cannot yield a higher throughput than $t_1^*$. 
    Consider a feasible policy $(\hat{t}_1, \hat{P}_{\mathrm{comm}}(\cdot))$ achieving a throughput $\hat{R}$. Since $t_1^* < \hat{t}_1$, the computation phase completes earlier, expanding the feasible time window for communication. 
    Let $a^*=t_{\mathrm{req}}+t_1^*+t_2$ and $\hat a=t_{\mathrm{req}}+\hat t_1+t_2$. By extending $\hat P_{\mathrm{comm}}(t)$ as zero over $[a^*,\hat a)$, the resulting policy remains feasible for (P1) and achieves the same throughput $\hat R$.
    Since reducing $t_1$ does not shrink the feasible set of communication policies, $t_1^*$ is optimal.
\end{proof}

\begin{remark}[Separation Principle]
    Proposition \ref{prop t1} reveals that the optimization of $t_1$ and $P_{\mathrm{comm}}(\cdot)$ can be decoupled without compromising optimality. This separation significantly reduces the problem complexity, as it allows the optimal waiting time to be determined prior to solving for the power control.
\end{remark}

Having determined the optimal waiting time $t_1^*$, we proceed to optimize the communication policy. Let $[a, b]$ denote the communication interval corresponding to $t_1^*$. We define the feasible set for the transmit power $P_{\mathrm{comm}}(\cdot)$ as
\begin{equation}
\begin{aligned}
    \mathcal{B} \triangleq \Big\{ & P_{\mathrm{comm}} : [a,b] \to [0,\infty) \;\Big|\; \\[-1.2ex]
    &\quad \int_{a}^{t} P_{\mathrm{comm}}(\tau) \, \mathrm{d}\tau \le \bar{E}(t), \quad \forall\, t \in [a,b] \Big\},
\end{aligned}
\label{feasible set}
\end{equation}
where \(\bar{E}(t) \triangleq  E(t) - P_{\mathrm{comp}} t_2\) represents the residual harvested energy available for communication. 
Since the task execution window $T$ is relatively small, the channel gain over the window can be assumed to be unchanged, i.e., $g(t) \equiv g$. 
For simplicity, we let \(\alpha \triangleq  \frac{g}{N_0 B}\),
\(\phi(P) \triangleq  B \log_2(1 + \alpha P)\). Note that $a$ is a linear function of $t_1$.
Consequently, (P1) reduces to the following problem:
\begin{equation}
    \mathrm{(P2)} \quad \max_{P_{\mathrm{comm}}(\cdot) \in \mathcal{B}} \quad 
    \int_{a}^{b} \phi\big( P_{\mathrm{comm}}(\tau) \big) \, \mathrm{d}\tau. \label{P3}
\end{equation}
Since $E_{\mathrm{comm}}(t)$ and $P_{\mathrm{comm}}(t)$ have a one-to-one mapping, we use them interchangeably in the analysis. 
The property of optimal transmit power is given by the following theorem.

\begin{theorem}[Optimal Policy via Lower Convex Envelope~\cite{varan2014energy}]\label{theorem 1}
    The optimal cumulative energy consumption $E^*_{\mathrm{comm}}(t)$ for $t \in [a, b]$ is given by the lower convex envelope (LCE) of $\bar{E}(t)$, anchored at the initial point $(a, 0)$ and the final point $(b, \bar{E}(b))$.
\end{theorem}

\subsection{Optimal Waiting Time and Transmit Power}

Based on the preceding results, this subsection introduces the optimal solutions for (P1) under the representative solar-EH profile in 
\eqref{EH energy}. According to Proposition~\ref{prop t1} and Theorem~\ref{theorem 1}, a minimal possible waiting time and an LCE of the solar-EH profile are required.
Leveraging the closed-form expression of the EH curve, we can determine a closed-form solution efficiently. For analytical clarity, we assume that $P_{\mathrm{comp}}\,t_2<\frac{P_0}{\omega}$, $t_2<T<\frac{\pi}{2\omega}$, and $E_0=0$. Other cases are either infeasible or can be handled using the same procedures.
To provide a unified expression for transmit power when $0\leq t_{\mathrm{req}}\leq \frac{2\pi}{\omega}-T$, we first investigate two specific cases for $t_{\mathrm{req}}$ as follows.

\subsubsection{Case 1: Task Starts from Orbital Noon}

We consider the scenario when $t_{\mathrm{req}}=0$ and derive the exact solution to (P1).
We start with finding $t^*_1$.
Since $E(t)$ is concave over $[0,\frac{\pi}{2\omega}]$, tangency between $E(t)$ and the computation-related segment is impossible. Therefore, the minimal feasible \(t_1\) occurs in either case as follows: (1) \emph{Immediate start}.
It is feasible if and only if \(\frac{P_0}{\omega}\sin(\omega t_2)\ \ge\ P_{\mathrm{comp}}\,t_2\),
and renders
$t_1^{\ast}\ =\ 0$. (2) \emph{Terminal touch} at \(t_1+t_2\).
It is feasible if \(\frac{P_0}{\omega}\sin(\omega t_2)\ <\ P_{\mathrm{comp}}\,t_2\).
In this case, a positive wait is necessary, and $t_1^{\ast}$ is characterized by tightness at the end of the computation window: \(E(t_1^{\ast}+t_2)=P_{\mathrm{comp}}t_2
\Longleftrightarrow
\sin\big(\omega(t_1^{\ast}+t_2)\big)=\frac{\omega P_{\mathrm{comp}}}{P_0}t_2\),
which yields
\begin{equation}
t_1^{\ast}\ =\ \frac{1}{\omega}\arcsin\!\Big(\frac{\omega P_{\mathrm{comp}}}{P_0}\,t_2\Big)\ -\ t_2,\label{t1 when t0=0}
\end{equation}
well-defined when $P_{\mathrm{comp}}t_2<\frac{P_0}{\omega}$ and $t_1^{\ast}\in\big[0,\frac{\pi}{2\omega}-t_2\big]$.
In summary, the optimal $t_1^{\ast}$ is given by
\begin{equation}
t_1^{\ast}\!=\!\begin{cases}
0, &\!\!\!\!\!  \frac{P_0}{\omega}\sin(\omega t_2) \!\ge\! P_{\mathrm{comp}}t_2,\\[-0.2ex]
 \frac{1}{\omega}\!\arcsin(\frac{\omega P_{\mathrm{comp}}}{P_0}t_2)\!-\!t_2, & \!\!\!\!\! \frac{P_0}{\omega}\sin(\omega t_2)\!<\!P_{\mathrm{comp}}t_2.
\end{cases}
\end{equation}

For $P_{\mathrm{comm}}^{\ast}(t)$, since $\bar E(t)$ is concave on $[a,b]$, its LCE is the affine function interpolating $(a,0)$ and $\big(b,\bar E(b)\big)$. Hence,
\begin{equation}
E^*_{\mathrm{comm}}(t)=\frac{t-a}{b-a}\,\bar E(b)=\frac{t-a}{b-a}\,\big[E(b)-P_{\mathrm{comp}}\,t_2\big].
\end{equation}
Taking the derivative yields the constant optimal power:
\begin{equation}
P_{\mathrm{comm}}^{\ast}(t)\!\equiv\!\frac{E(b)\!-\!P_{\mathrm{comp}}\,t_2}{b-a}=\frac{\tfrac{P_0}{\omega}\sin(\omega b)\!-\!P_{\mathrm{comp}}\,t_2}{b-a}.
\label{constant power policy}
\end{equation}
For some parameters, the optimal policy is shown in Fig.~\ref{case 1}.\footnote{Since the actual time window is too narrow to clearly observe power variations, we employ an extended time window to better visualize the solution. This scaling is intended purely for qualitative illustration rather than rigorous quantitative analysis. This also applies to Fig.~\ref{case 2}.}

\begin{figure}[t]
\centering
\subfigure[$t_{\mathrm{req}}=0$.]{
    \includegraphics[width=0.32\textwidth]{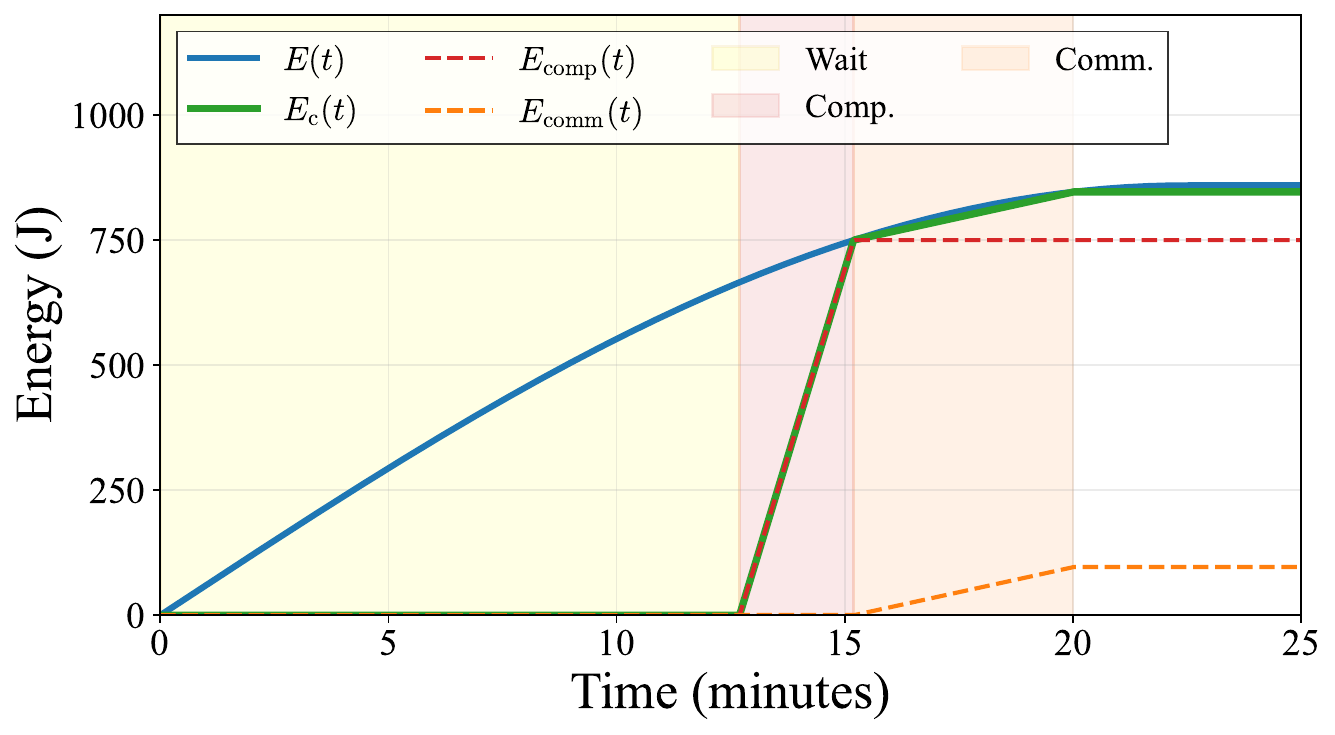}
    \label{case 1}
}
\hfill
\subfigure[$t_{\mathrm{req}}=\frac{3\pi}{2\omega}$.]{
    \includegraphics[width=0.32\textwidth]{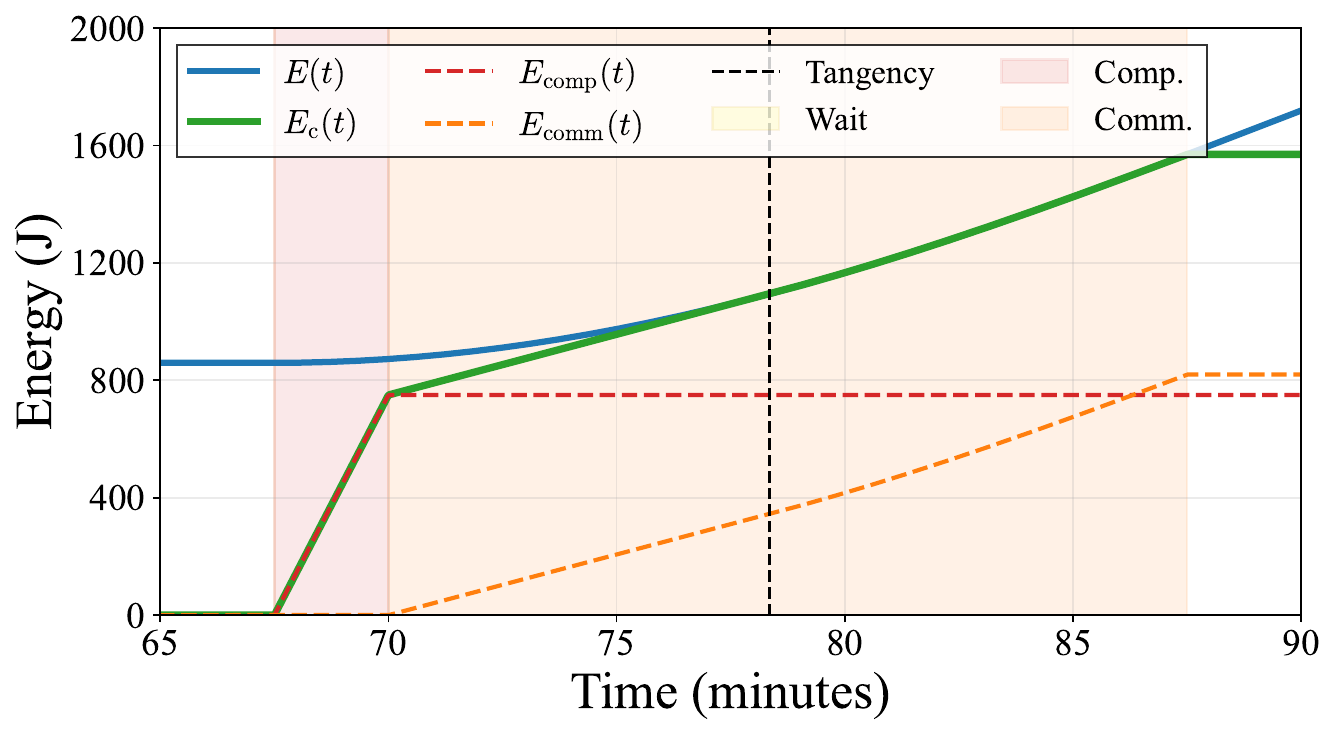}
    \label{case 2}
}
\caption{Energy evolution and time allocation under the optimal policy for two representative task-request times. 
Here, $E_c(t)=E_{\mathrm{comp}}(t)+E_{\mathrm{comm}}(t)$ denotes the total cumulative energy consumption.}
\captionsetup{justification=justified}
\label{fig:case_combined}
\vspace{-3ex}
\end{figure}

\subsubsection{Case 2: Task Starts from Orbital Dawn}

We consider the scenario when $t_{\mathrm{req}}=\frac{3\pi}{2\omega}$.
Under the assumption $P_{\mathrm{comp}}\,t_2<\tfrac{P_0}{\omega}$, the energy-causality constraint is satisfied even with $t_1=0$. Hence
$t_1^{\ast}\ =\ 0.$ 
For $P_{\mathrm{comm}}^{\ast}(t)$, we first investigate a tangent construction.
Define the slope of the supporting line anchored at $(a,0)$ as
\begin{equation}
s\ \triangleq \ \inf_{t\in(a,\,b]}\ \frac{\bar E(t)-0}{t-a}
\ =\ \inf_{t\in(a,\,b]}\ \frac{E(t)-P_{\mathrm{comp}}t_2}{t-a}.
\end{equation}
Since $\bar E$ is convex on $(a,\,b]$, the infimum is achieved at a unique point $\tau\in(a,\,b]$. If no interior tangency exists, the minimizer occurs at the boundary, and we set $\tau\triangleq b$.
When an interior tangency exists, $\tau$ satisfies the equal-slope condition \(\frac{\bar E(\tau)-0}{\tau-a}\ =\ \bar E'(\tau)\label{eq:tangent-condition}\).
With explicit form of $\bar E(t)$, it becomes
\begin{equation}
\frac{\tfrac{P_0}{\omega}\big[2
+\sin(\omega \tau)\big]-P_{\mathrm{comp}}t_2}{\tau-a}\ =\ P_0\cos(\omega \tau).
\label{eq:tangent-explicit}
\end{equation}
After transforming it to $P_0\big[2+\sin(\omega \tau)\big]-\omega P_{\mathrm{comp}}t_2-P_0\omega(\tau-a)\cos(\omega \tau)=0,$ it follows that the left-hand side is monotone; the equation can therefore be solved efficiently by bisection.
Therefore, $E_{\mathrm{comm}}^{\ast}(t)$ and the corresponding $P_{\mathrm{comm}}^{\ast}(t)$ are
\begin{equation}
E_{\mathrm{comm}}^{\ast}(t)\ =\
\begin{cases}
\displaystyle \frac{t-a}{\tau-a}\,\bar E(\tau), & t\in[a,\tau],\\[-0.2ex]
\bar E(t), & t\in[\tau,\,b],
\end{cases}
\end{equation}
\begin{equation}
P_{\mathrm{comm}}^{\ast}(t)=\begin{cases}
\displaystyle \frac{\bar E(\tau)}{\tau-a}=\frac{E(\tau)-P_{\mathrm{comp}}t_2}{\tau-a}, &\! t\in[a,\tau],\\[-0.2ex]
\bar E'(t)=P_0\cos(\omega t), &\! t\in(\tau,\,b].
\end{cases}
\label{eq:policy-piecewise}
\end{equation}
Note that when $t\in(\tau,\,b]$, $P_{\mathrm{comm}}^{\ast}(t)$ follows a \emph{cosine-power control}. If no interior tangency exists (i.e., $\tau=b$), $P_{\mathrm{comm}}^{\ast}(t)$ reduces to a constant-power policy $P_{\mathrm{comm}}^{\ast}(t)\equiv\frac{\bar E\!(b)}{b-a}$.
For some parameters, the optimal policy is shown in Fig.~\ref{case 2}.

\subsubsection{Optimal Solutions}

Based on the previous discussions of the two special cases, we have identified the basic structure of the solutions, which naturally extends to the general case.
We first examine $t^*_1$ w.r.t. different $t_{\mathrm{req}}$. From \eqref{t1 when t0=0}, define 
\begin{equation}
    \tau_0\triangleq \frac{1}{\omega}\arcsin\!\Big(\frac{\omega P_{\mathrm{comp}}}{P_0}\,t_2\Big)\ -\ t_2.\label{tau0}
\end{equation}
The causality constraint requires $t_{\mathrm{req}}+t_1\ge \tau_0$. If $\tau_0>0$, it implies $t_1\ge \tau_0-t_{\mathrm{req}}$. Otherwise, if $\tau_0\le 0$, then $t_{\mathrm{req}}+t_1\ge \tau_0$ always holds. Therefore, we have
\begin{equation}
    t_1^{\ast} = \max\{0,\ \tau_0 - t_{\mathrm{req}}\}.\label{t^*_1 general}
\end{equation}

Next, we examine $P_{\mathrm{comm}}^{\ast}(t)$. When $t_{\mathrm{req}}\leq \tfrac{3\pi}{2\omega}-T$, $E(t)$ is concave over $[a,b]$, constant power applies: $P_{\mathrm{comm}}^{\ast}(t)\equiv\frac{\bar E\!(b)}{b-a}$. When $\tfrac{3\pi}{2\omega}-T\leq t_{\mathrm{req}}\leq \tfrac{2\pi}{\omega}-T$, $E(t)$ is convex over $[a,b]$, and the tangent construction (case 2) applies, yielding \eqref{P^*_comm general}.
\begin{figure*}[!t]
\begin{equation}
P_{\mathrm{comm}}^{\ast}(t)\ =\
\begin{cases}
\displaystyle \frac{E(b)-P_{\mathrm{comp}}\,t_2}{\,b-a\,},\ \ t\in[a,b], \ \  & \text{if } t_{\mathrm{req}}\le \tfrac{3\pi}{2\omega}-T,\\[0pt]
\begin{cases}
\displaystyle \frac{E(\tau)-P_{\mathrm{comp}}\,t_2}{\,\tau-a\,}, & t\in[a,\tau],\\[-0.5ex]
\displaystyle P(t), & t\in(\tau,b),
\end{cases}
& \text{if } \tfrac{3\pi}{2\omega}-T\le t_{\mathrm{req}}\le \tfrac{2\pi}{\omega}-T.
\label{P^*_comm general}
\end{cases}
\end{equation}
\vspace{-5ex}
\end{figure*}

\subsection{Discussion}

\subsubsection{The Meaning of $\tau_0$}

The quantity \(\tau_0\) is a virtual threshold time defined by the
endpoint-tightness condition: \(E(\tau_0+t_2)=P_{\mathrm{comp}}t_2\).
It indicates the earliest computation-start threshold implied by the
harvested-energy curve. If \(\tau_0>0\), starting before \(\tau_0\)
violates the computation energy-causality constraint, and the satellite
must wait until \(\tau_0\). If \(\tau_0\le 0\), this threshold lies before
the considered task horizon, so immediate computation is already feasible.
Thus, \eqref{t^*_1 general} holds.

\subsubsection{Constant-Power Control Versus Cosine-Power Control}
\label{constant vs cosine}

The optimal power structure is determined by the convexity of \(E(t)\) over \([a,b]\). If \(E(t)\) is concave, i.e., \(t_{\mathrm{req}}\le \frac{3\pi}{2\omega}-T\), the LCE is the chord from \((a,0)\) to \((b,\bar E(b))\), yielding constant-power control. If \(E(t)\) is convex, i.e., \(\frac{3\pi}{2\omega}-T\le t_{\mathrm{req}}\le \frac{2\pi}{\omega}-T\), the policy is constant on \([a,\tau]\) and then follows cosine-power control, i.e., \(P_{\mathrm{comm}}^\ast(t)=P(t)=P_0\cos(\omega t)\) on \((\tau,b)\). This cosine-power control is a distinctive feature of the considered solar-EH-driven space AI system, where the optimal transmit power can track the orbital solar-EH profile. Nevertheless, constant-power control applies in most cases, covering approximately three quarters of the task-start range.

\subsubsection{Extension to a General Solar-EH Profile}

Although the previous subsection evaluates a representative profile for demonstration, the properties established in Sec.~\ref{solution structure} hold for any solar-EH profile. For example, for an arbitrary solar panel normal vector $\mathbf{n} = [n_x, n_y, n_z]^T$, the solar-EH profile is
\begin{equation}
\begin{split}
    P(t) = &\eta \cdot \gamma \cdot A \cdot \max( -n_x \cos\beta \sin(\omega t)\autobreak{}\quad- n_y \sin\beta - n_z \cos\beta \cos(\omega t), 0).
\end{split}\label{output power}
\end{equation}
See Appendix \ref{derivation of P(t)} for the derivation. For a sanity check, substituting $\mathbf{n}=[0,0,-1]^T$ into \eqref{output power} recovers \eqref{EH power}. The corresponding $E(t)$ can be obtained by integrating $P(t)$. The solution derivations under such profiles are omitted for brevity. 

\section{Joint C$^2$ Utility Optimization under Solar EH}\label{joint C2}
\label{Utility Maximization}

 In this section, we aim to optimize the energy and time resources between computation and communication (C$^2$) under solar EH.

\subsection{Problem Formulation}

\subsubsection{Metric Design}

We consider two critical yet inherently conflicting objectives for space generative AI (see Section~\ref{metrics}): (onboard) CLIP score and communication throughput. 
To evaluate the overall performance, we define the \emph{joint C$^2$ utility} as a composite metric of these two terms:
\begin{equation}
    U = S + \lambda R,
\end{equation}
where $\lambda$ is a balancing coefficient. 

\subsubsection{Approximation of CLIP Score}

\begin{figure}[t]
\centering
{\includegraphics[width=0.32\textwidth]{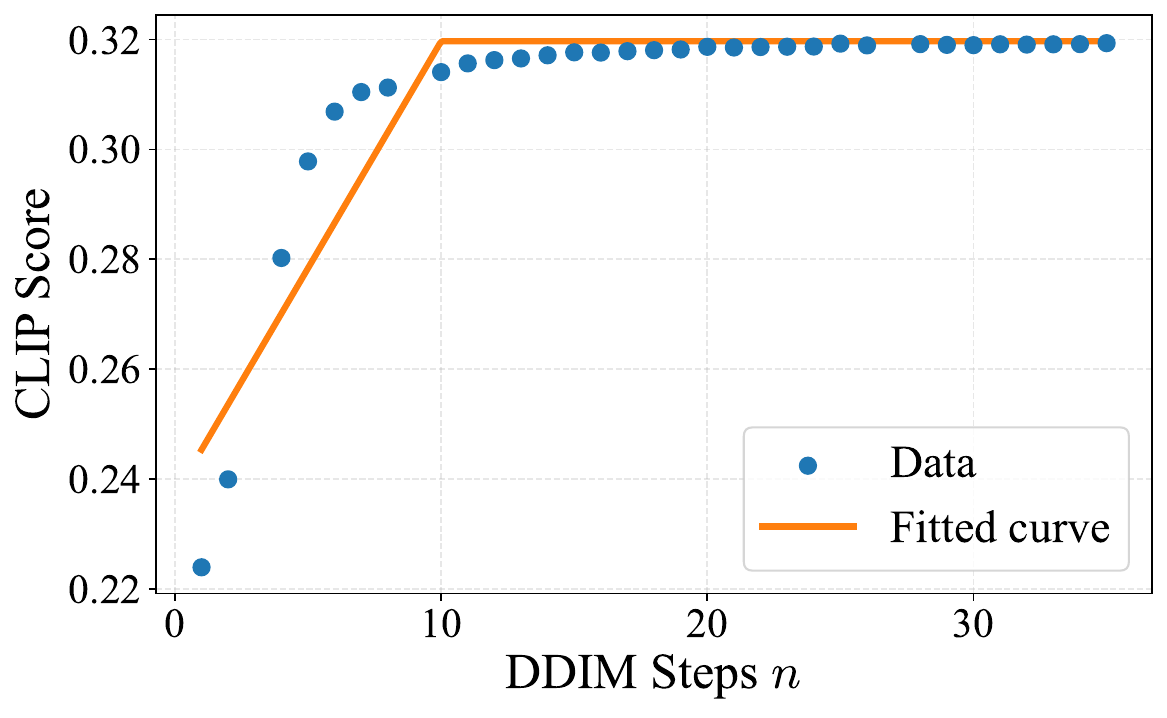}}
\caption{CLIP score versus DDIM steps.}
\captionsetup{justification=justified}
\label{s(n)}
\vspace{-2ex}
  \end{figure}

As shown in Fig.~\ref{s(n)}, the CLIP score increases rapidly at early DDIM steps, while additional steps provide only marginal gains. Motivated by this saturation behavior, we restrict the DDIM step $n$ to the effective set $\mathcal{N}\triangleq
\{n_{\min},n_{\min}+1,\ldots,n_{\max}\}$, and approximate the onboard CLIP score within this set as
\begin{equation}
    S(n) \approx c_3 n + c_4,\label{approx S(n)}
\end{equation}
where $c_3>0$ and $c_4$ are fitting coefficients.

\subsubsection{Problem Description}

We aim to maximize the joint C$^2$ utility by optimizing the DDIM steps, waiting time, and power control under the given solar-EH profile. A critical constraint is energy causality, ensuring that total energy consumption never exceeds the cumulative harvested energy. Let $a(n) \triangleq t_{\mathrm{req}} + t_1 + t_2(n)$ denote the start of the communication phase. 
Accordingly, we formulate the joint optimization over the DDIM step
$n\in\mathcal{N}$, waiting time $t_1$, and transmit-power policy
$P_{\mathrm{comm}}(\cdot)$ as follows:
\begin{subequations}
\begin{align}
\mathrm{(P3)}\max_{n,t_1, P_{\mathrm{comm}}(\cdot)}  &\;
c_3 n + c_4+\lambda\int_{a(n)}^{b}
\phi\!\big(P_{\mathrm{comm}}(t)\big)\, \mathrm{d}t\label{opt goal P2}
\\[0.8ex]
\text{s.t.}  &
\;P_{\mathrm{comp}}(t - t_{\mathrm{req}} - t_1) \le E(t),
 \\ &\qquad \forall t \in [t_{\mathrm{req}}+t_1, a(n)],\label{comp constraint P2}
\\[0.6ex]
 &
\int_{a(n)}^{t} \!\!\!\!\!\!P_{\mathrm{comm}}\!(\tau) \mathrm{d}\tau \!\le\! E\!(t)\!-\!P_{\mathrm{comp}} t_2\!(n), \\&\qquad \forall t \in [a(n), b].\label{comm constraint P2}
\end{align}\label{P2}
\end{subequations}

\subsection{Problem Decomposition}

The main difficulty of (P3) lies in the coupling among computation, communication, and solar EH. Computation and communication draw from the same solar-EH-powered energy buffer, creating a time-causal resource conflict between the two stages. To this end, we decompose (P3) with respect to the DDIM step and characterize the resulting optimality structure in the following proposition.

\begin{proposition}[Decomposition of (P3)]
\label{prop:P3_decomposition}
For any fixed \(n\in\mathcal{N}\), let \(R^{\ast}(n)\) denote the maximum throughput obtained by solving the communication-throughput maximization problem (P1) with \(t_2=t_2(n)\). Then, the optimal DDIM step of (P3) is given by
\begin{equation}
    n^{\ast}
    =    \arg\max_{n\in\mathcal{N}}
    \left\{
    c_3n+c_4+\lambda R^{\ast}(n)
    \right\}.
\end{equation}
\end{proposition}

\begin{proof}
For each fixed $n$, (P3) reduces to (P1); maximizing $c_3n+c_4+\lambda R^*(n)$ over $n\in\mathcal{N}$ yields the result.
\end{proof}

This decomposition reduces the original mixed discrete-continuous problem to an outer one-dimensional search over \(n\in\mathcal{N}\) and an inner communication-throughput maximization problem. The resulting search over \(\mathcal{N}\), referred to as \emph{exhaustive search}, yields an exact solution to (P3). More importantly, the decomposition separates the role of computation from the communication-control subproblem, thereby providing the basis for the closed-form analysis in the sequel.

\subsection{Closed-form Analysis}

Although exhaustive search can solve (P3) exactly over the discrete set \(\mathcal N\), such a search-based solution provides limited insight into how the optimal DDIM step $n$ depends on system parameters. It also requires repeatedly solving the inner communication problem for all feasible $n$. To reveal the structure of the \(C^2\) trade-off and enable lightweight onboard adaptation, we analyze the dominant constant-power regime, in which the reduced problem admits a closed-form characterization. Specifically, we consider the following problem:
\begin{equation}
    \mathrm{(P4)} \quad\max_{\,n\in\mathcal{N}} \quad c_3 n + c_4 + \lambda  R(n),\label{P4}
\end{equation}
where \(R(n) = \big(b-a(n)\big)\,\phi\big(P_{\mathrm{comm}}(n)\big)\). Here, \(P_{\mathrm{comm}}(n) \;=\; \frac{E_b - P_{\mathrm{comp}}\, t_2(n)}{\,b-a(n)\,}\) with \(E_b\triangleq E(b)\).
For analytical tractability, we relax \(n\) from the discrete set \(\mathcal N\) to the interval \([n_{\min},n_{\max}]\), and later recover the integer solution by evaluating the nearest feasible DDIM steps. We set \(t_1=0\) and impose the feasibility conditions \(E_{\mathrm{comp}}(a(n_{\max}))<E(a(n_{\max}))\) and \(a(n)<b\), which ensure sufficient computation energy and a non-empty communication window, respectively. 

We first propose the following proposition to describe the properties of $R(n)$. 

\begin{proposition}[Properties of $R(n)$]\label{prop R(n)}
    $R(n)$ is concave and decreases monotonically w.r.t. $n$.
\end{proposition}

\begin{proof}
    See Appendix \ref{proof prop 2}.
\end{proof}

    The monotonicity of $R(n)$ reflects the fundamental resource trade-off: increasing  $n$ consumes more energy, thereby reducing the power available for transmission. Furthermore, the concavity of $R(n)$ implies that the marginal degradation in throughput accelerates as energy becomes scarcer.

\begin{corollary}[Concavity of $U(n)$] \label{cor:U_concavity}
    The utility function $U(n)$ is concave w.r.t. $n$.
\end{corollary}

\begin{proof}
Since $S''(n)=0$ and $\lambda>0$,
\[
U''(n)=\lambda R''(n)\le0.
\]
Hence, $U(n)$ is concave.
\end{proof}

Fig.~\ref{tradeoff} illustrates the result of Corollary~\ref{cor:U_concavity} and the fundamental \(C^2\) trade-off. As the DDIM step \(n\) increases, \(S(n)\) improves due to greater computational investment. However, this consumes more energy and reduces the budget available for transmission, thereby degrading the throughput \(R(n)\). As a result, \(U(n)\) exhibits unimodal behavior, increasing initially before declining. The maximizer therefore represents the optimal balance between computation and communication, and shifts as the weighting parameter \(\lambda\) varies. Corollary~\ref{cor:U_concavity} enables the following closed-form
characterization.
\begin{figure}[t]
\centering
{\includegraphics[width=0.32\textwidth]{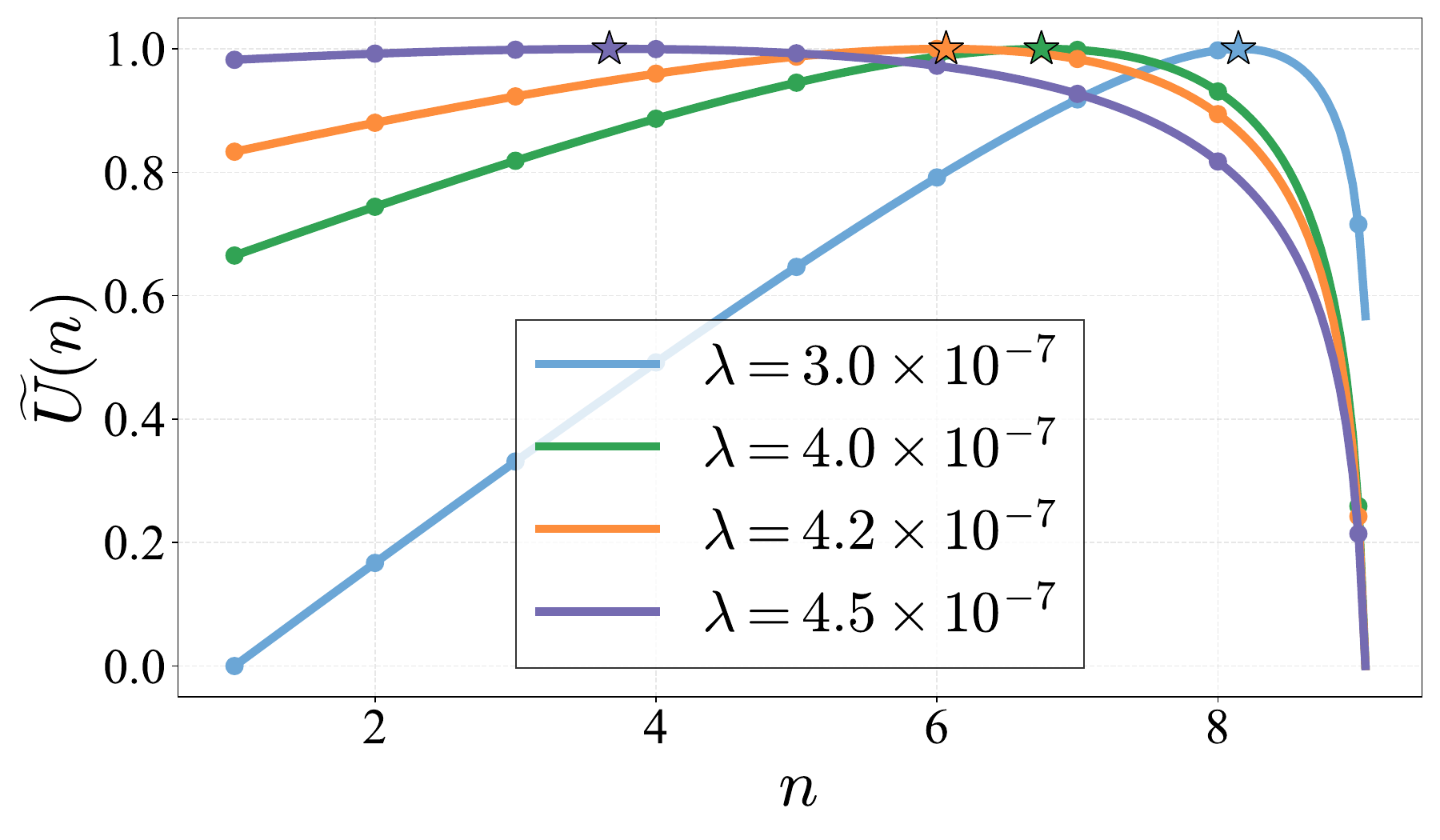}}
\caption{The \(C^2\) trade-off. The normalized utility, \(\tilde{U}(n)\), is plotted.}
\captionsetup{justification=justified}
\label{tradeoff}
\vspace{-2ex}
  \end{figure}

\begin{theorem}[Closed-form Step Selection] \label{thm:optimal_n}
    In the dominant constant-power regime, under the feasibility conditions stated above and assuming that a real stationary point exists, equation $U'(n)=0$ admits the following solution: \begin{equation}\label{eq:optimal_n_cont}
    \tilde{n} = 
    \begin{cases}
        a_1 + a_2\frac{\alpha \Delta_E}{D} \frac{ W_0\bigl(-D e^{-C}\bigr) }{ W_0\bigl(-D e^{-C}\bigr) + 1 }, & \text{if } \Delta_E > 0, \\[0pt]
        a_1 + a_2\frac{\alpha \Delta_E}{D} \frac{ W_{-1}\bigl(-D e^{-C}\bigr) }{ W_{-1}\bigl(-D e^{-C}\bigr) + 1 }, & \text{if } \Delta_E < 0,
    \end{cases}
    \end{equation}
    where $a_1 \triangleq \frac{T - t_1 - c_2}{c_1}$, $a_2 \triangleq \frac{1}{c_1}$, $C \triangleq 1 + \frac{a_2c_3}{\lambda B} \ln 2$, $D \triangleq 1 + \alpha P_{\mathrm{comp}}$, and $\Delta_E \triangleq  E_b - \Pcomp (T-t_1)$. $W_k(\cdot)$ denotes the Lambert $W$ function with branch $k$~\cite{corless1996lambert}.
    The solution $n^*$ to problem (P4) is given by:
    \(n^* = \operatorname*{arg\,max}_{n \in \mathcal{S}} U(n),\)
    where \(\mathcal{S} \triangleq \left\{\lfloor \left[ \tilde{n}  \right]_{n_{\min}}^{n_{\max}}\rfloor, \; \lceil \left[ \tilde{n} \right]_{n_{\min}}^{n_{\max}} \rceil \right\}\) is the candidate set.
    Here, $[x]_a^b \triangleq \min(\max(x, a), b)$ denotes the clipping operator. $\lfloor \cdot \rfloor$ and $\lceil \cdot \rceil$ denote the floor and ceiling functions, respectively. The degenerate case $\Delta E=0$, for which $U(n)$ becomes affine, is omitted for brevity.
\end{theorem}

\begin{proof}
    See Appendix~\ref{app:proof_thm2}.
\end{proof}

The condition $\Delta_E \gtrless 0$ creates a physical boundary equivalent to $P_{\mathrm{comm}} \gtrless P_{\mathrm{comp}}$, which can be verified by
    \[\Delta_E \gtrless 0 \iff E_b \gtrless P_{\mathrm{comp}}(T-t_1) \iff P_{\mathrm{comm}} \gtrless P_{\mathrm{comp}}.\]
    This distinction is pivotal because the Lambert $W$ function branches, $W_0$ and $W_{-1}$, exhibit opposing monotonicity. Consequently, \(n^*\) responds differently to system parameters in the communication- and computation-dominant power regimes.

\begin{remark}[Complexity Analysis]
The closed-form rule in Theorem~\ref{thm:optimal_n} reduces online decision-making to evaluating at most two candidate integers around \(\tilde n\). Therefore, its online complexity is \(O(1)\). In contrast, the exact exhaustive-search solution requires evaluating all feasible DDIM steps in \(\mathcal N\), leading to a complexity of \(O(|\mathcal N|)\). 
\end{remark}

\section{Experimental Results}
\label{Experimental Results}

\subsection{Experimental Settings}

\begin{table}[htbp]
\small
    \centering
    \caption{System Parameters}
    \label{tab:simulation_parameters}
    \renewcommand{\arraystretch}{0.88}
    \begin{tabular}{llc}
        \toprule
        \textbf{Parameter} & \textbf{Symbol} & \textbf{Value} \\
        \midrule
        Orbital angular velocity & $\omega$ & $2\pi/5400$ rad/s \\
        Computation power & $P_{\text{comp}}$ & $300$ W \\
        Channel bandwidth & $B$ & $5 \times 10^3$ Hz \\
        Noise power spectral density & $N_0$ & $-174$ dBm/Hz \\
        Additional attenuation factor & $L_{\mathrm{AL}}$ & $-5$ dB \\
        Transmit/receive antenna gain & $G_{\mathrm{tx}}$ and $G_{\mathrm{rx}}$  & $53$ dB and $0$ dB \\   
        Carrier frequency & $f_c$  & $12$ GHz\\
        Task execution window & $T$ & $3$ s \\
        Energy conversion efficiency & $\eta$ & $0.19$ \\
        Solar irradiance constant & $\gamma$ & $1353$ W/m$^2$ \\
        Solar panel area & $A$ & $2.0$ m$^2$ \\
        Weighting parameter & $\lambda$ & $4.3869 \times 10^{-7}$ \\
        Minimum DDIM step & $n_{\text{min}}$ & $4$ \\
        Maximum DDIM step & $n_{\text{max}}$ & $10$ \\
        Initial energy at task arrival & $E_0$ & $400$ J \\
        \bottomrule
    \end{tabular}
\end{table}

\begin{enumerate}
    \item \textit{System and communication settings:} We investigate a space generative AI framework, and the relevant system parameters are summarized in Table \ref{tab:simulation_parameters}. The communication and energy-harvesting settings largely follow the specifications in~\cite{chen2024fedmeld} and~\cite{yang2016towards}, respectively.
    \item \textit{Metrics:} To evaluate the proposed framework, we adopt the following two metrics: (1) \emph{Communication Throughput}: It is defined as the total number of bits delivered over the transmission window [a,b], and is used to quantify communication performance. (2) \emph{E2E CLIP Score}: We measure the semantic similarity between the original text prompt and the image received at the ground device\cite{radford2021learning}, to indicate the overall E2E generation quality. The reported metrics are averaged across diverse channel realizations and multiple queries.
    \item \textit{Benchmarking schemes:} We consider two categories of benchmarks. The first category is used to validate the proposed solver, while the second category evaluates the E2E performance gains of our joint \(C^2\) control.
\begin{itemize}
    \item \emph{Solver benchmark:} We compare the proposed closed-form solution with exhaustive search over the feasible set \(\mathcal{N}=\{n_{\min},\ldots,n_{\max}\}\). It verifies that the closed-form solution guarantees optimality while avoiding the \(O(|\mathcal{N}|)\) search complexity.
    \item \emph{Computation-centric scheme:} It enforces a fixed, high DDIM step $n=n_\text{max}$ to prioritize image fidelity, potentially at the expense of transmission reliability in energy-scarce regimes.
    \item \emph{Communication-centric scheme:} It prioritizes energy allocation for the downlink transmission to ensure high throughput, while only guaranteeing the minimal required DDIM step $n=n_\text{min}$.
\end{itemize}
\end{enumerate}

\subsection{Validation of Throughput Maximization (P1)}

\begin{figure}[t]
\centering
\subfigure[Average communication throughput versus $A$.]{\includegraphics[width=0.32\textwidth]{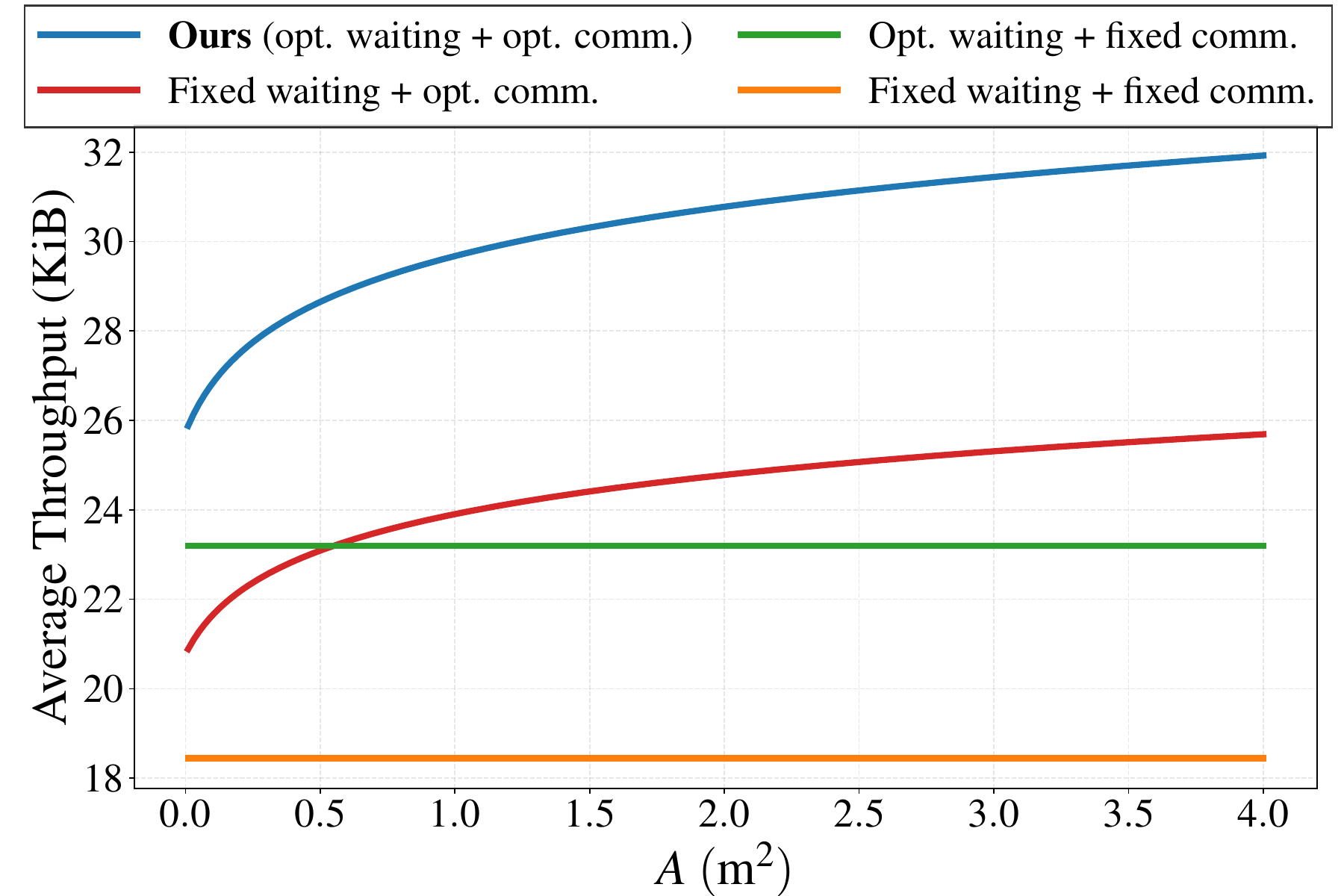}\label{throughput A}}
\subfigure[Average communication throughput versus $\beta$.]{\includegraphics[width=0.32\textwidth]{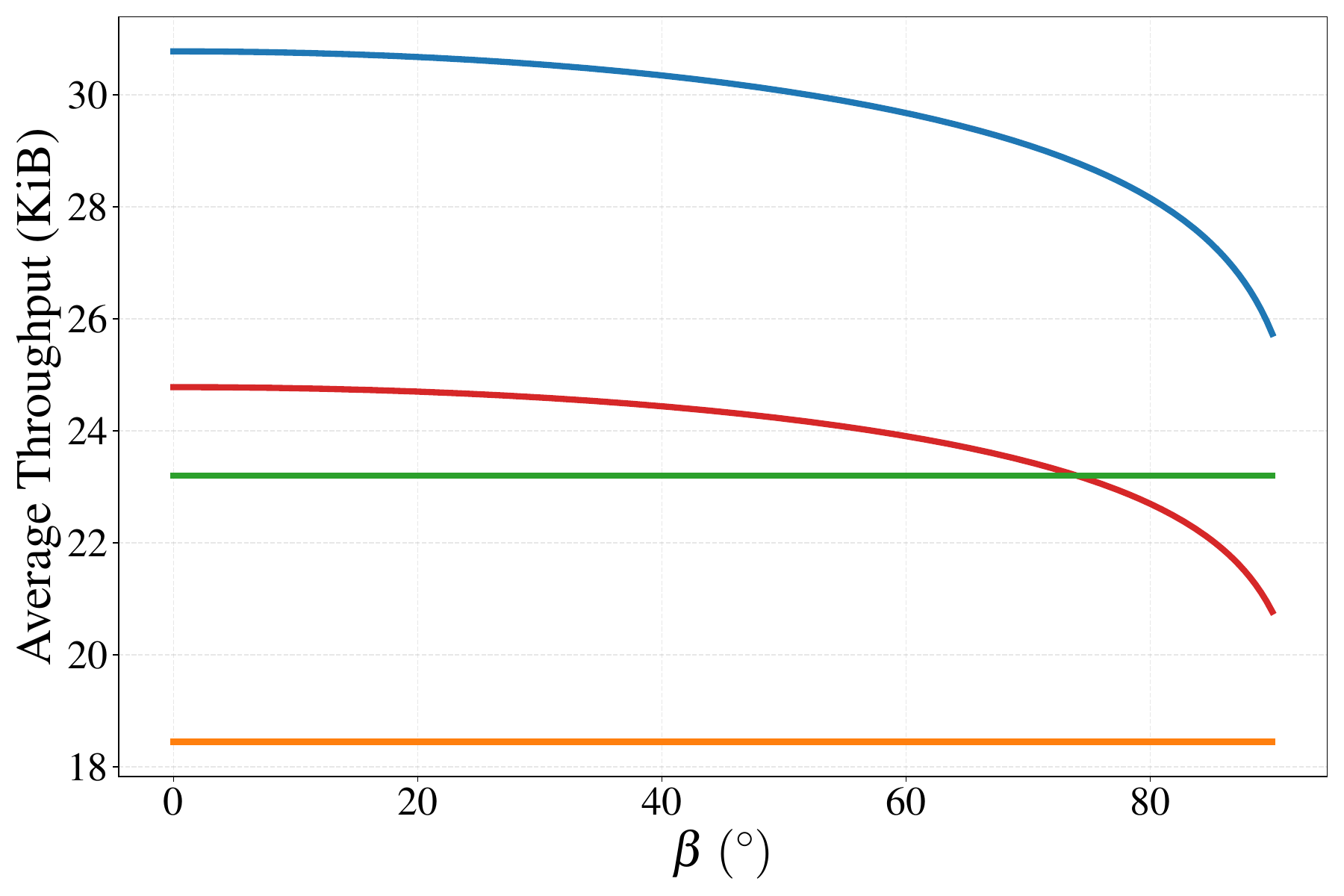}\label{throughput beta}}
\caption{Communication throughput across different schemes.}\label{fig:throughput}
\vspace{-10pt}
\end{figure}

To isolate the gains derived from the optimal waiting time and transmit power, in this subsection, we fix the DDIM step to $n=5$ and investigate the communication throughput. The following three baselines are considered: (1) \emph{Fixed waiting time and optimal transmit power.} $t_1 = 0.4$~s, with the optimal power obtained from Theorem~\ref{theorem 1}. (2) \emph{Optimal waiting time and fixed transmit power.} $t_1^*$ is obtained from Proposition~\ref{prop t1}, with $P_{\mathrm{comm}}=10$ W. (3) \emph{Fixed waiting time and fixed transmit power.} Both the waiting time and transmit power are fixed: $t_1 = 0.4$~s and $P_{\text{comm}} = 10$~W.

As illustrated in Fig.~\ref{fig:throughput}, our proposed strategy consistently achieves the highest communication throughput across the entire range of panel areas $A$ and solar elevation angles $\beta$. In Fig.~\ref{throughput A}, the throughput grows logarithmically with $A$, demonstrating that our joint optimization of $t_1^*$ and $P_{\text{comm}}^*(t)$ effectively converts increased harvested energy into higher data volume. In Fig.~\ref{throughput beta}, we observe that throughput remains stable at low elevation angles but degrades as $\beta$ approaches $90^\circ$, reflecting the reduction in the effective solar harvesting area as the sun vector becomes orthogonal to the orbital plane. The persistent performance gap between our scheme and the baselines demonstrates the necessity of minimizing the waiting phase and adopting the LCE principle.

\subsection{Evaluation of the Closed-form Solution for (P4)}

\begin{figure}[t]
\centering
\subfigure[$n^*$ versus $\lambda$.]{\includegraphics[width=0.24\textwidth]{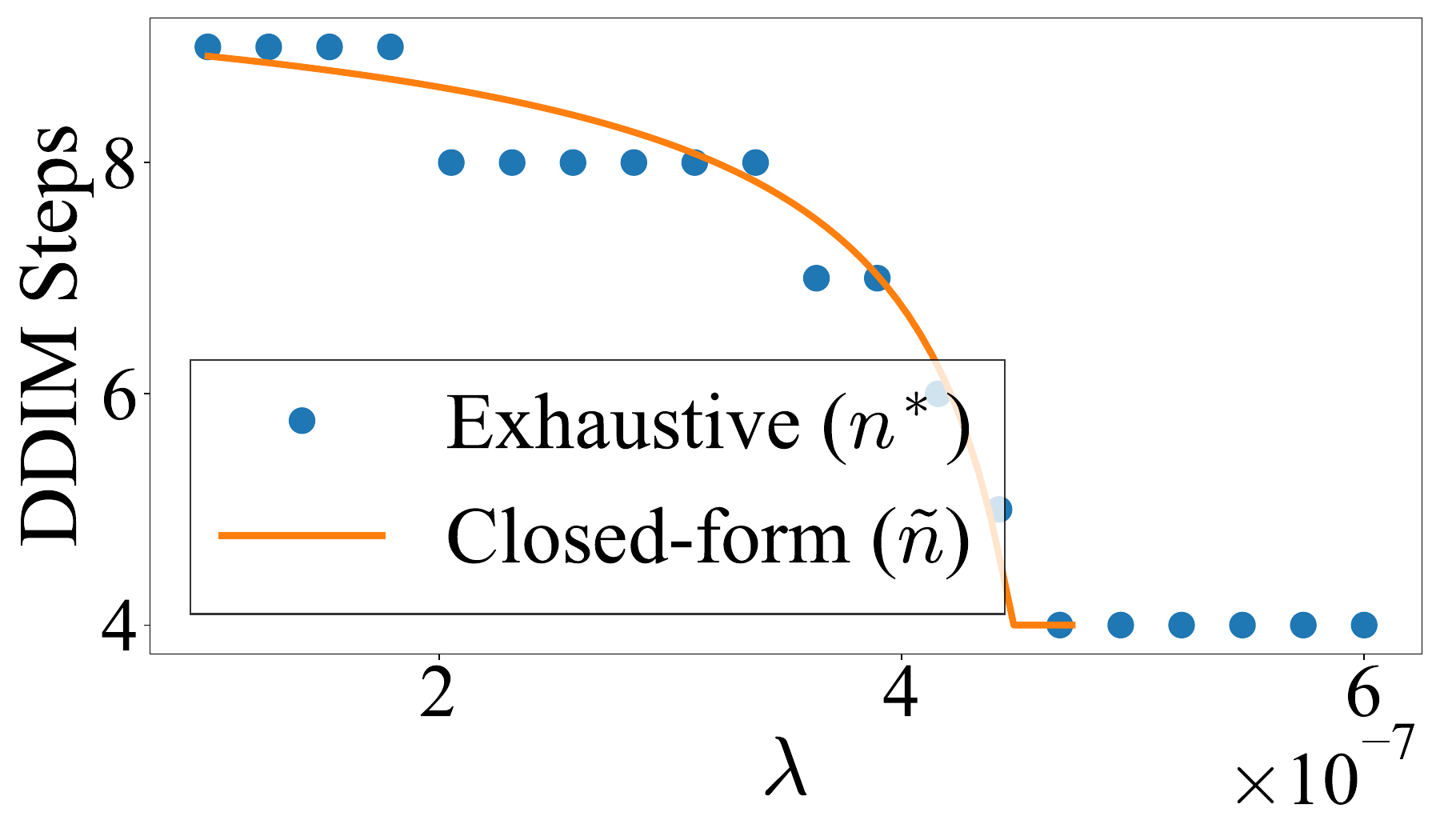}\label{n lam}} 
\subfigure[$n^*$ versus $E_0$.]{\includegraphics[width=0.24\textwidth]{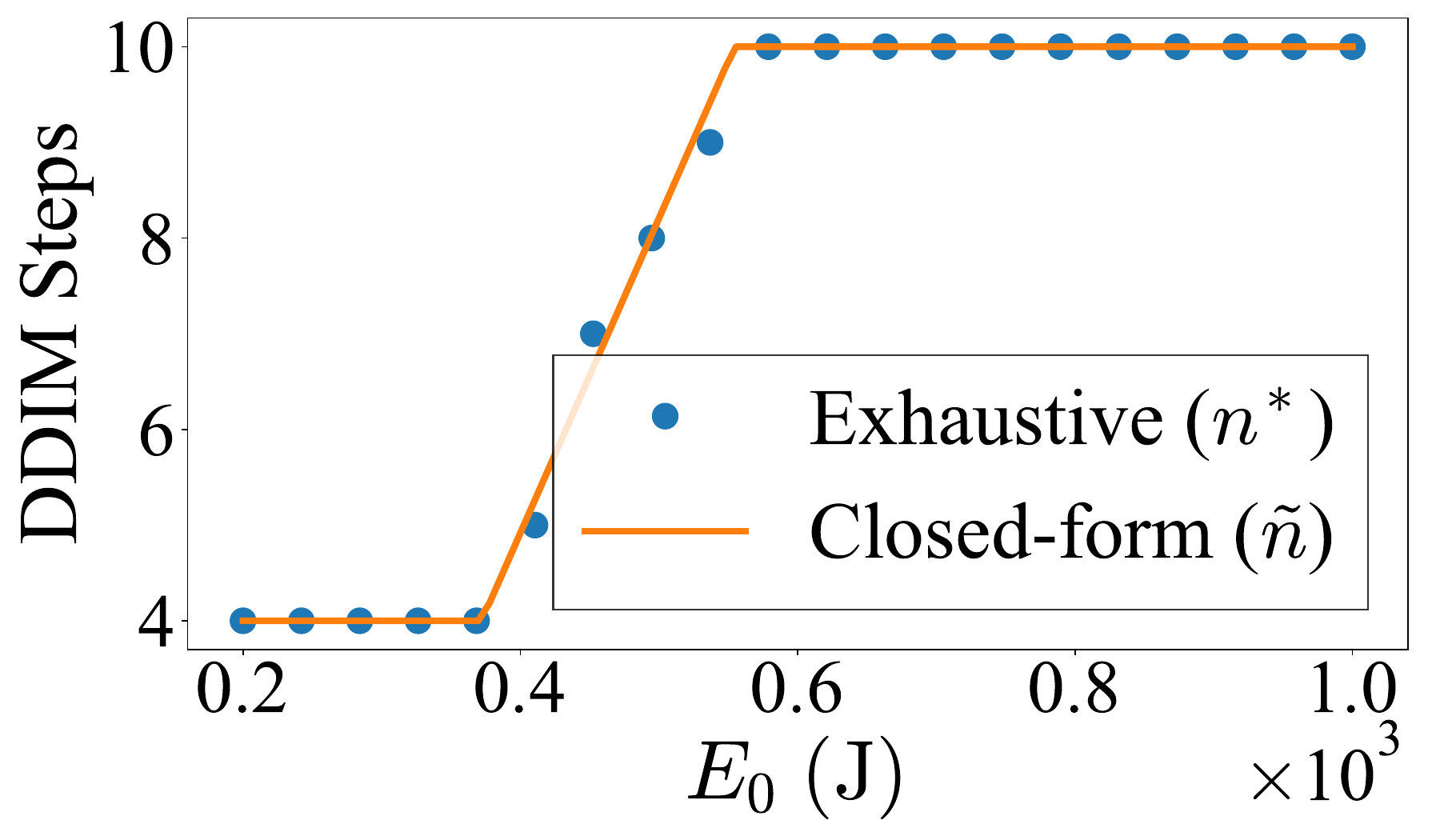}\label{n E0}}  
\subfigure[$n^*$ versus $B$.]{\includegraphics[width=0.24\textwidth]{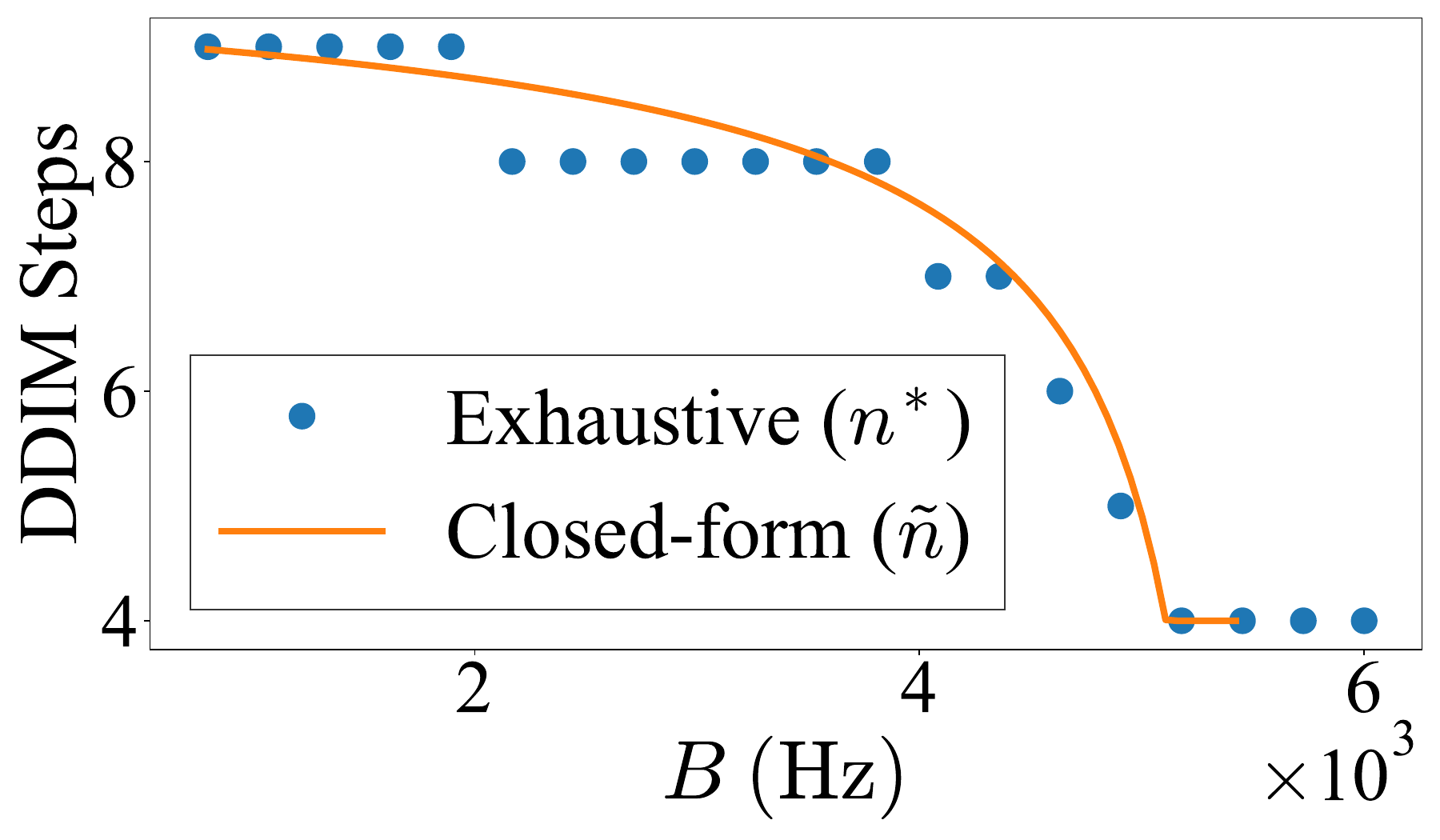}\label{n B}}  
\subfigure[$n^*$ versus $T$.]{\includegraphics[width=0.24\textwidth]{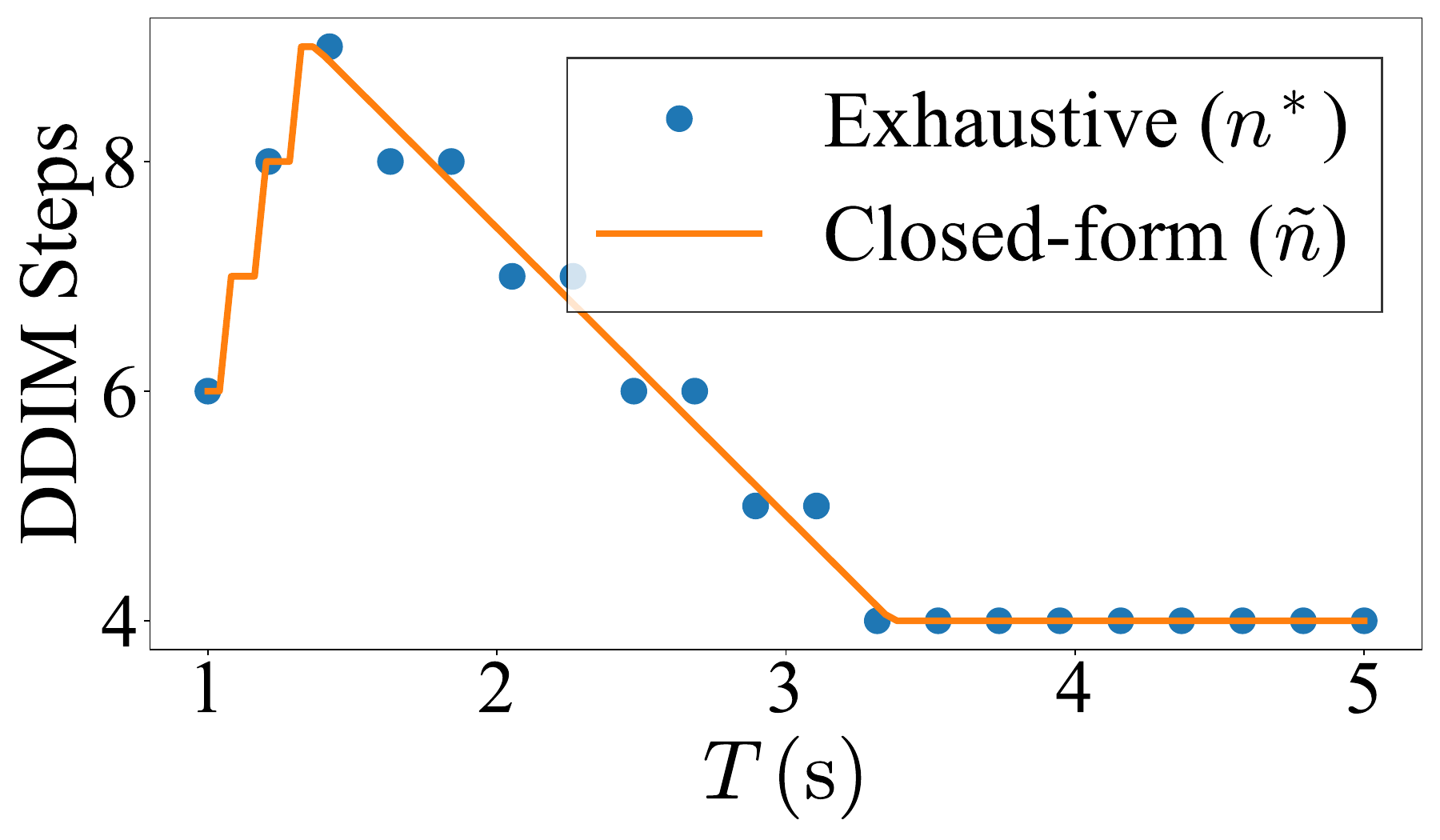}\label{n T}}  
\caption{Comparison of the optimal DDIM steps obtained by exhaustive search and the pre-rounding closed-form solution under different system parameters.}\label{comparison ex cls}
  \end{figure}

As shown in Fig.~\ref{comparison ex cls}, the pre-rounding closed-form solution closely tracks the exhaustive-search optimum across the considered settings.
This demonstrates that the closed-form rule maintains solution accuracy while reducing complexity. Analyzing the communication-dominant scenario reveals key resource trade-offs. First, $n^*$ decreases as the throughput weight $\lambda$ increases (Fig.~\ref{n lam}), shifting resources toward transmission. A similar trend is observed for the bandwidth \(B\) (Fig.~\ref{n B}). Conversely, higher initial energy $E_0$ allows for more computation, increasing $n^*$ (Fig.~\ref{n E0}). 
Interestingly, $n^*$ exhibits a non-monotonic relationship with the deadline $T$ (Fig.~\ref{n T}). Under tight deadlines ($T < 1.5$~s), transmission is severely bottlenecked, so the system prioritizes computation to preserve generation quality. As $T$ relaxes further, $n^*$ steadily decreases to fully exploit the extended window for throughput. 

\subsection{Performance Evaluation}

In this subsection, we evaluate the holistic performance of our joint C$^2$ optimization scheme. 

\subsubsection{Comparison with Other Schemes}

\begin{figure}[t]
\centering
\subfigure[Average E2E CLIP score versus $P_0$.]{\includegraphics[width=0.34\textwidth]{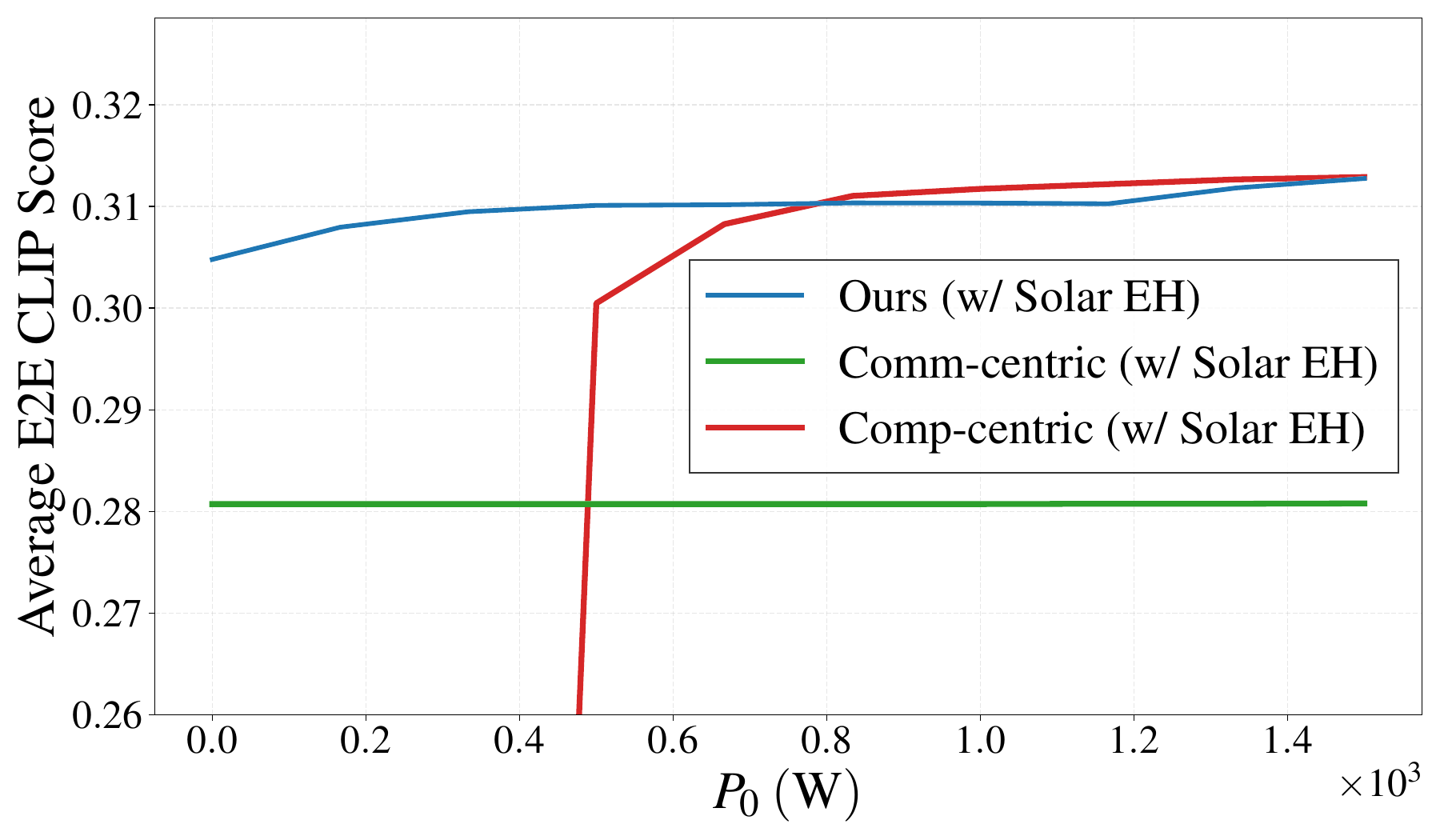}\label{CLIP_P0_w/_EH}}
\subfigure[Average E2E CLIP score versus $E_0$.]{\includegraphics[width=0.34\textwidth]{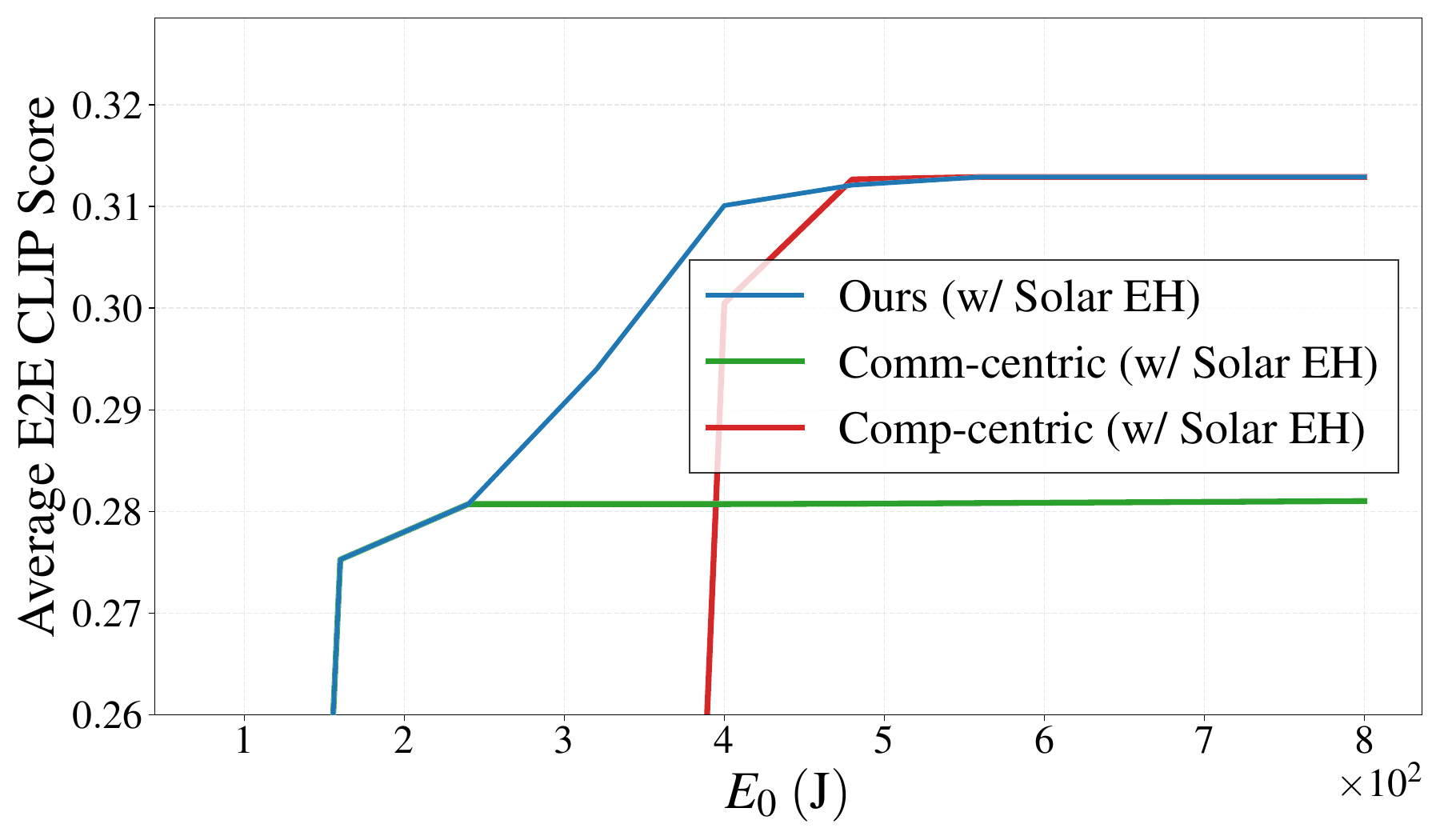}\label{CLIP_E0_w/_EH}}
\caption{Performance comparison of various schemes under solar EH.}\label{E2E_CLIP_Score_System_Params}
\end{figure}

As shown in Fig. 7, the proposed joint C$^2$ policy achieves the highest E2E CLIP score across the considered energy range. In the low-energy regime, it reduces the DDIM step and approaches the communication-centric baseline, preserving sufficient energy for transmission. In the high-energy regime, it allocates more resources to generation and approaches the computation-centric baseline. The largest gain appears in the moderate-energy regime, where neither extreme allocation is optimal. The proposed policy effectively balances onboard generation quality and downlink throughput. 
The sharp transition of each scheme occurs when the available energy crosses its scheme-dependent feasibility threshold for completing both generation and transmission. Below the corresponding threshold, insufficient energy remains for downlink transmission after generation, resulting in transmission failure.

\subsubsection{Impact of Orbital Position}

\begin{figure}[t]
\centering
\subfigure[Average communication throughput.]{\includegraphics[width=0.32\textwidth]{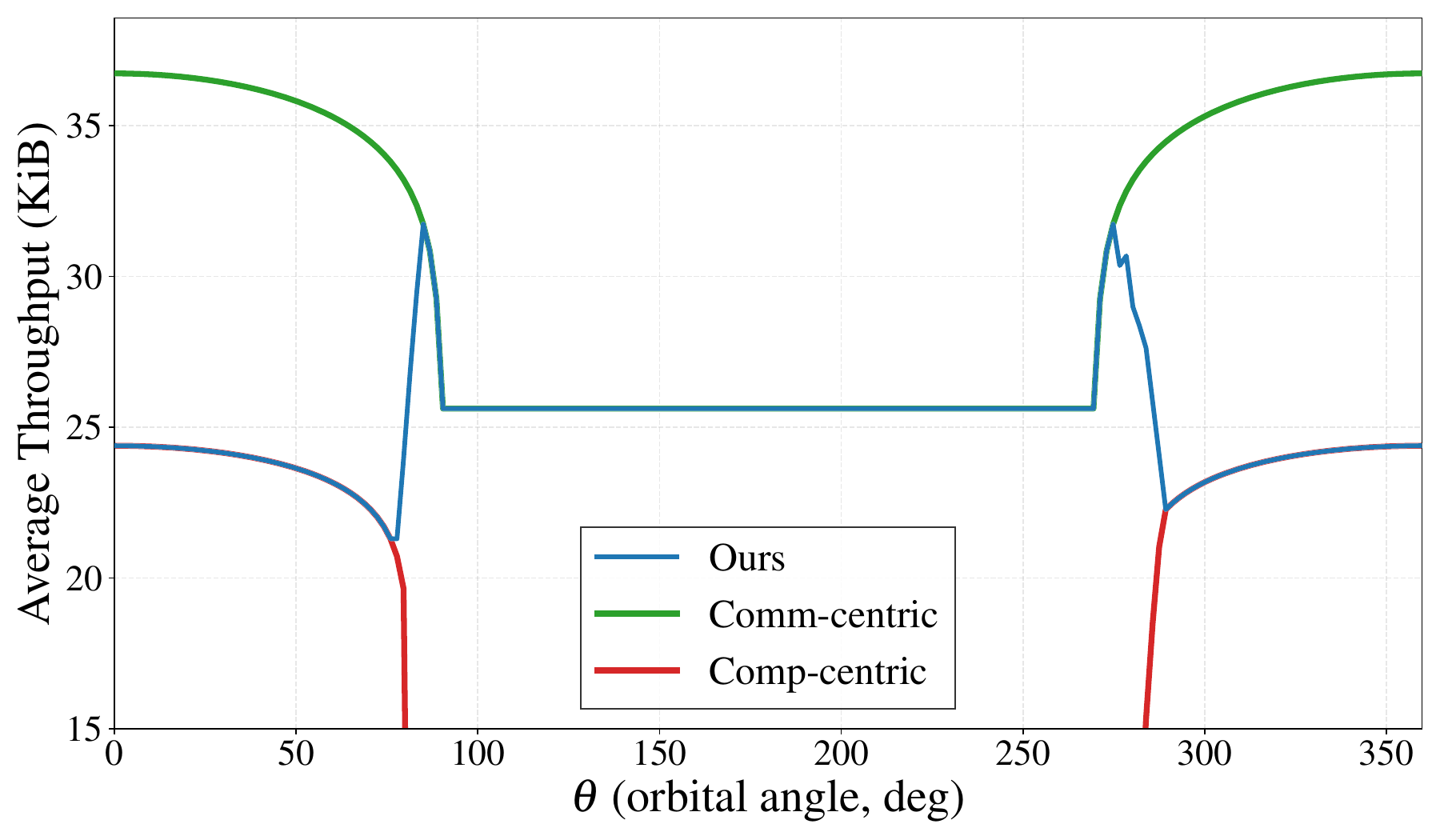}\label{fig:throughput_theta}}
\subfigure[Average E2E CLIP score.]{\includegraphics[width=0.32\textwidth]{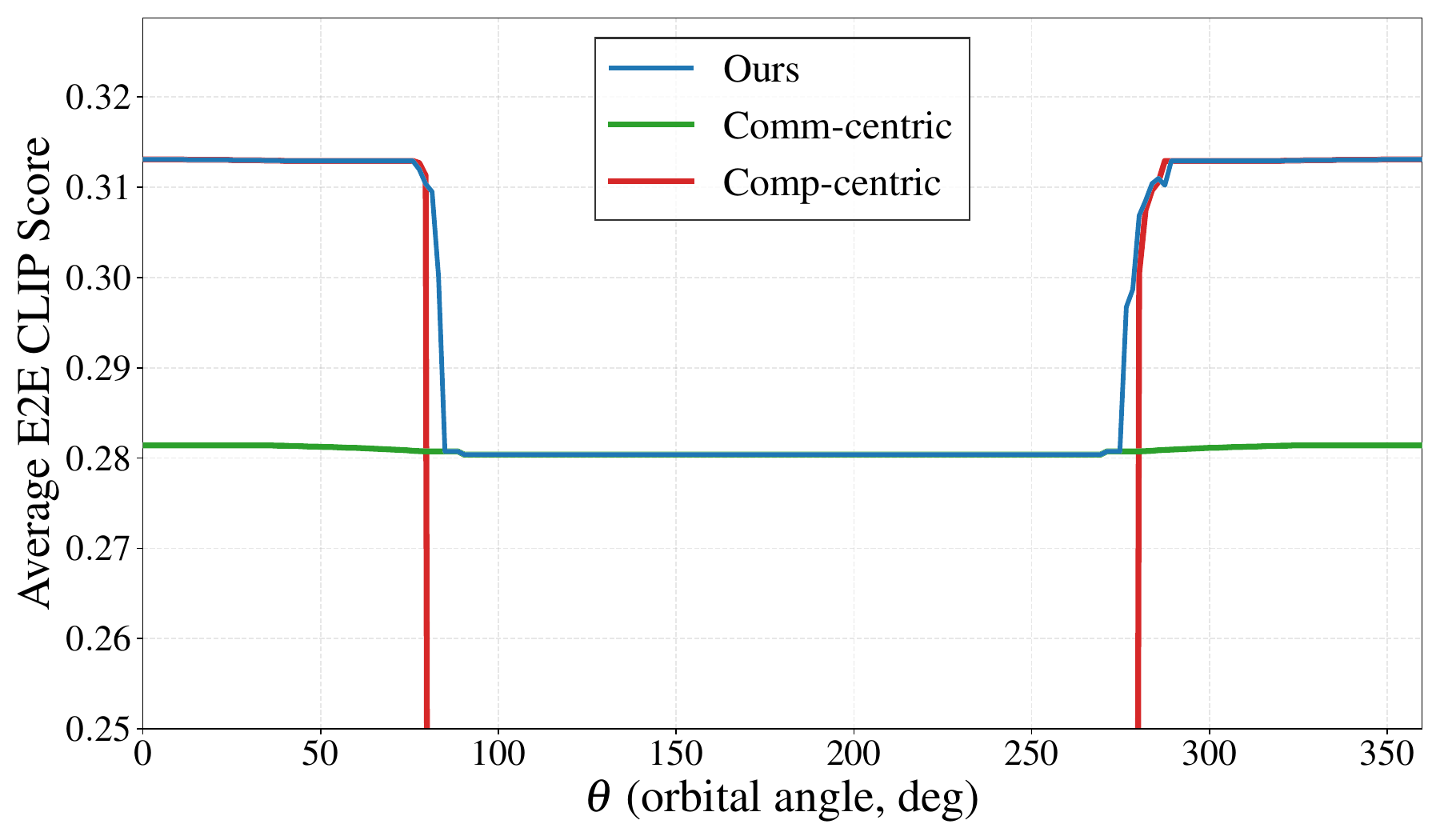}\label{fig:clip_theta}}
\caption{System performance at different orbital angles ($\theta \in [0^\circ, 360^\circ]$).}\label{fig:orbital_dynamics}
\end{figure}

Fig.~\ref{fig:orbital_dynamics} illustrates the communication throughput and E2E CLIP score for independent task executions initiated at different orbital angles, each with initial energy $E_0$ and the corresponding solar-EH profile.
Although the communication-centric baseline achieves the highest throughput by minimizing the computation load, its E2E CLIP score remains lower than the other schemes in sunlit regions because the onboard generation quality is limited by the minimum DDIM step.
In contrast, the computation-centric baseline achieves a high E2E CLIP score in sunlit regions, but suffers a sharp performance drop during the eclipse. This is because the fixed high DDIM step consumes excessive computation energy when no solar energy is harvested, leaving insufficient energy for downlink transmission. As a result, transmission fails and the E2E CLIP score drops to zero.
The proposed joint \(C^2\) policy avoids both drawbacks by adapting the DDIM step to the orbital energy state: it allocates more resources to generation when sunlight is available and scales back computation during the eclipse to preserve downlink transmission. Consequently, it maintains the highest E2E CLIP score over the full orbit.

\subsubsection{Impact of Solar EH}

\begin{figure}[t]
\centering
\includegraphics[width=0.32\textwidth]{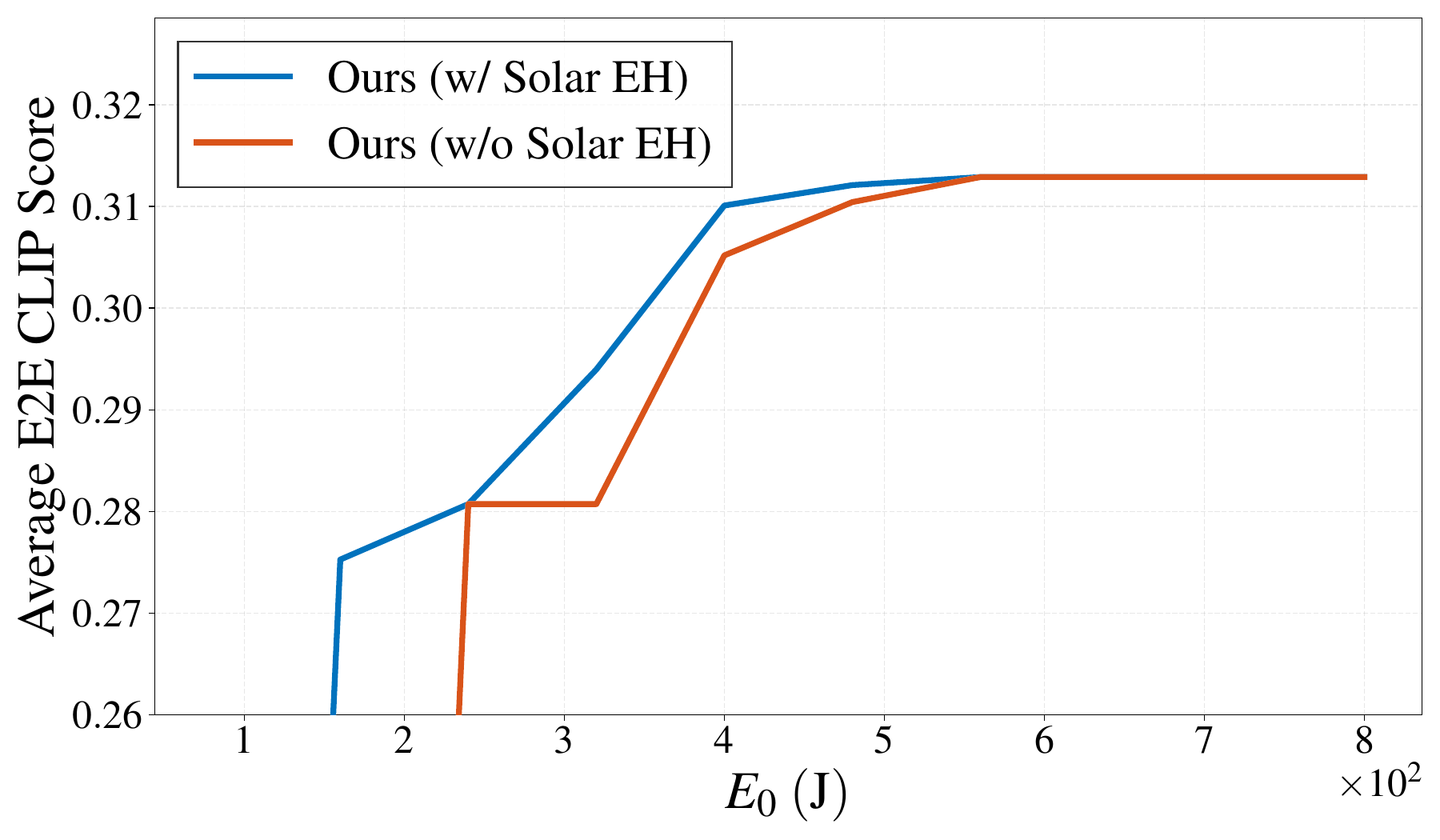}
\caption{Impact of solar EH on the average E2E CLIP score.}\label{fig:clip_eh_compare}
\end{figure}

\begin{figure}[t]
\centering
\includegraphics[width=0.32\textwidth]{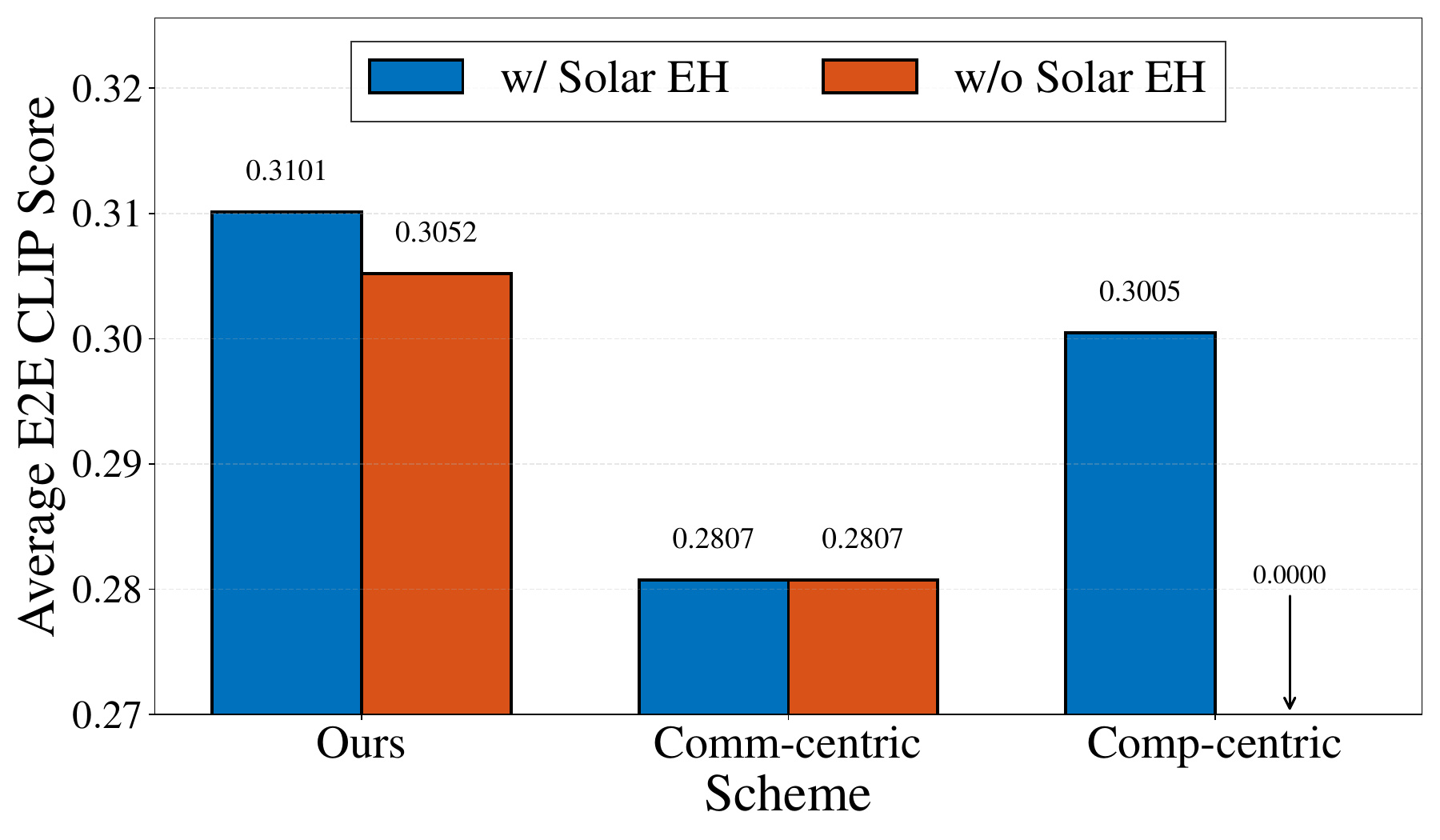}
\caption{Comparison of the average E2E CLIP score of different schemes in a specific scenario, evaluated with and without solar EH.}\label{fig:clip_bar}
\end{figure}

\begin{figure}[t]
    \centering
    \subfigure[Visual performance with solar EH.]{        \includegraphics[width=0.42\textwidth]{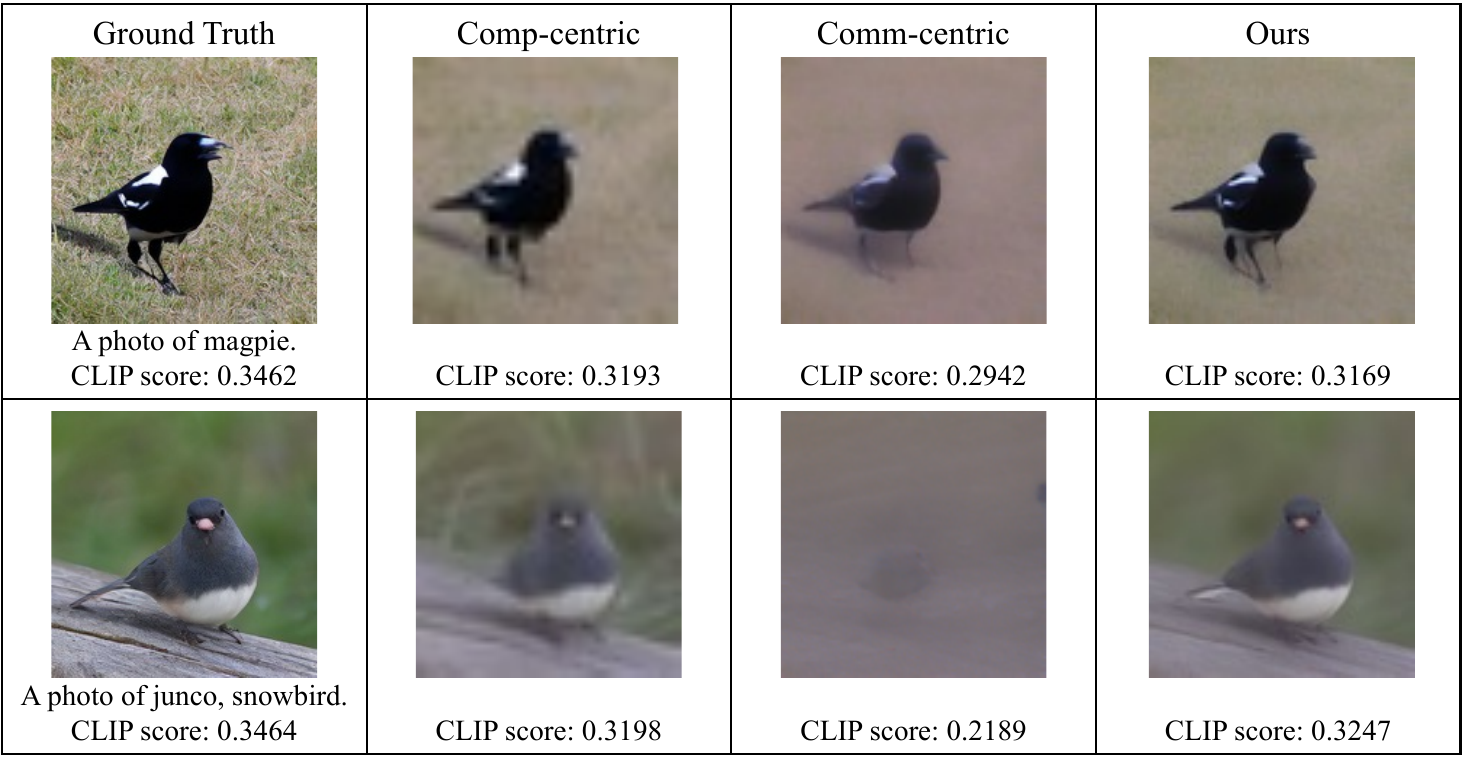}
        \label{fig:sample_with_eh}
    }
    \subfigure[Visual performance without solar EH.]{        \includegraphics[width=0.42\textwidth]{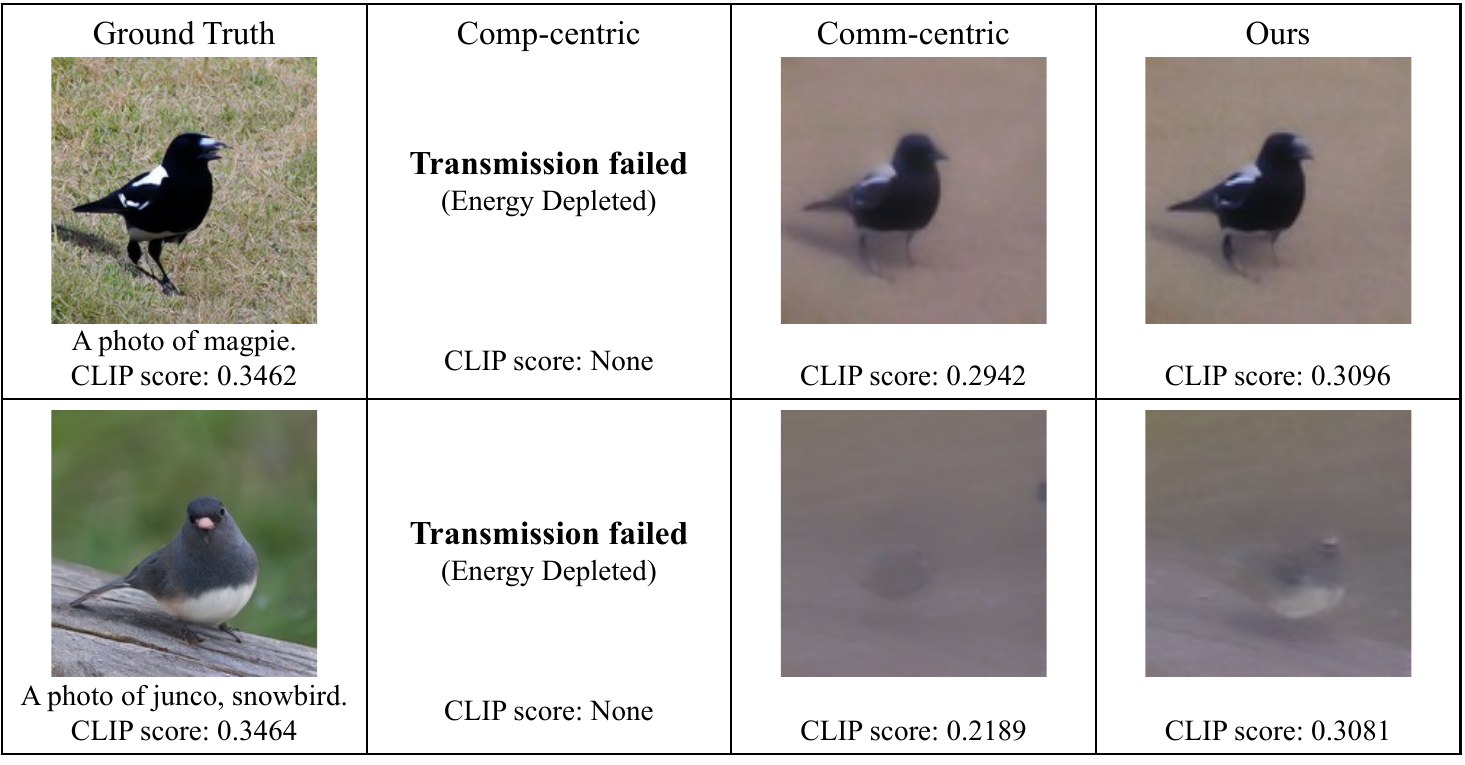}
        \label{fig:sample_without_eh}
    }

    \caption{Qualitative comparison of generated images.}
    \label{fig:visual_samples}
    \vspace{-5pt}
\end{figure}

Fig.~\ref{fig:clip_eh_compare} shows the impact of solar EH under different \(E_0\). In the resource-constrained regime, solar EH substantially improves the E2E CLIP score because the harvested energy helps satisfy the minimum generation-and-transmission requirement and can further support a larger DDIM step or higher downlink throughput. The visible step-like transitions are caused by discrete DDIM step selection: once the available energy exceeds a threshold that supports an additional DDIM step while preserving reliable transmission, \(n\) increases, resulting in a sudden improvement in the E2E CLIP score. As \(E_0\) becomes sufficiently large, the curves with and without solar EH converge because the initial energy alone is enough to support both generation and transmission, making additional harvested energy less influential.
Fig.~\ref{fig:clip_bar} further compares different schemes with and without solar EH in a representative scenario. With solar EH, the proposed scheme achieves the highest E2E CLIP score by jointly adapting computation and communication. Without solar EH, the computation-centric baseline becomes transmission-infeasible because its fixed high DDIM step depletes the energy budget before downlink transmission, causing transmission failure. The communication-centric baseline remains stable but quality-limited because it always prioritizes transmission and uses the minimum DDIM step. In contrast, the proposed scheme scales back computation when solar EH is unavailable, preserving transmission reliability while maintaining a competitive E2E CLIP score.
The qualitative examples in Fig.~\ref{fig:visual_samples} are consistent with these results: solar EH improves visual quality, and without solar EH, the proposed scheme still delivers semantically aligned images while the computation-centric baseline fails to transmit.

\section{Concluding Remarks}
\label{Concluding Remarks}

This work establishes a framework for space generative AI and reveals the inherent coupled \(C^2\) bottleneck under solar EH. Specifically, it characterizes how a satellite should jointly manage waiting, diffusion-based generation, and downlink transmission when \(C^2\) stages draw from the same time-causal harvested-energy buffer. First, computation should start at the earliest energy-feasible time. Second, constant transmit power is optimal in most orbital states, while dawn may require solar-tracking control. Third, adaptive generation depth is most beneficial under moderate energy, where neither computation- nor communication-centric extremes are optimal. These results advance satellite edge intelligence from generic computation-offloading design towards task-oriented generative service provisioning under physical orbital constraints.

Future work should extend this framework to multi-user and multi-task satellite systems, where heterogeneous generative requests compete for shared harvested energy. It is also important to incorporate practical hardware constraints such as finite battery capacity, peak transmit power, and accelerator-dependent inference energy. Finally, more refined E2E quality models are needed to jointly capture generation quality, compression distortion, channel effects, and human-perceived semantic fidelity.

\section*{Appendix}
\appendices
\renewcommand{\thesubsection}{\Alph{subsection}}

\subsection{Derivation for Eq.~\eqref{output power}}\label{derivation of P(t)}

As the satellite orbits the Earth with angular velocity $\omega$, the \(\mathcal{F}_b\) rotates relative to the \(\mathcal{F}_o\)  around the \(Y_o\). The rotation matrix from $\mathcal{F}_o$ to $\mathcal{F}_b$ at time $t$ is:
\begin{equation}
    \mathbf{R}(t) = 
    \begin{bmatrix}
    -\sin(\omega t) & 0 & -\cos(\omega t) \\
    0 & -1 & 0 \\
    -\cos(\omega t) & 0 & \sin(\omega t)
    \end{bmatrix}.
\end{equation}
The sun vector in \(\mathcal{F}_b\) is derived as $\mathbf{S}_b(t) = \mathbf{R}(t) \mathbf{S}_o$:
\begin{equation}
    \mathbf{S}_b(t) = 
    \begin{bmatrix}
    -\cos\beta \sin(\omega t) \\
    -\sin\beta \\
    -\cos\beta \cos(\omega t)
    \end{bmatrix}.
\end{equation}

The instantaneous harvested power is proportional to the effective projection area:
\begin{equation}
    P(t) = \eta \cdot \gamma \cdot A \cdot \max\left( \mathbf{n} \cdot \mathbf{S}_b(t), 0 \right).
\end{equation}
Taking the inner product yields the results.

\subsection{Proof of Proposition \ref{prop R(n)}}
\label{proof prop 2}

For $R(n)$, with $u\triangleq T-t_1-c_1n-c_2$ and $\Delta_E \triangleq  E_b - \Pcomp (T-t_1)$, it becomes
\begin{align}
f(u)
&= B\, u \, \log_2\!\Bigl( 1 + \alpha \Pcomp + \frac{\alpha \Delta_E}{u} \Bigr), \quad u>0.\label{f(u)}
\end{align}
Calculating its first derivative yields

\begin{equation} \label{f'(u)}
\begin{split}
    f'(u)
    &= B \Bigg[
    \log_2\!\Bigl( 1 + \alpha \Pcomp + \tfrac{\alpha \Delta_E}{u} \Bigr)
    \autobreak[-10pt]{}+
    \frac{u}{\ln 2}
    \underbrace{
        \frac{d}{du}\ln\!\Bigl( 1 + \alpha \Pcomp + \tfrac{\alpha \Delta_E}{u} \Bigr)
    }_{\textstyle =\, \frac{-\alpha \Delta_E/u^2}{1 + \alpha \Pcomp + \alpha \Delta_E/u}}
    \Bigg] \\[-5pt] 
    &= B \Bigg[
    \log_2\!\Bigl( 1 + \alpha \Pcomp + \tfrac{\alpha \Delta_E}{u} \Bigr)
    \;\autobreak[-7pt]{}-\;
    \frac{1}{\ln 2}\,
    \frac{\alpha \Delta_E}{u(1 + \alpha \Pcomp) + \alpha \Delta_E}
    \Bigg].
\end{split}
\end{equation}
The second derivative is
\begin{equation}
\begin{split}
    f''(u)
    &= B \frac{d}{du}\Bigg\{
    \log_2\!\Bigl( 1 + \alpha \Pcomp + \tfrac{\alpha \Delta_E}{u} \Bigr)
    \;\autobreak[-10pt]{}-\;
    \frac{1}{\ln 2}\,
    \frac{\alpha \Delta_E}{\,u\bigl(1 + \alpha \Pcomp + \alpha \Delta_E/u\bigr)}
    \Bigg\} \\[-5pt]
    &= \frac{B}{\ln 2} \Bigg[
    \frac{-\,\alpha \Delta_E/u^2}{\,1 + \alpha \Pcomp + \alpha \Delta_E/u\,}
    \;\autobreak[-5pt]{}-\;
    \frac{d}{du}
    \left( \frac{\alpha \Delta_E}{\,u\big(1 + \alpha \Pcomp + \alpha \Delta_E/u\big)} \right)
    \Bigg].
\end{split} 
\end{equation}
Compute the remaining derivative. Let
  $\psi(u) \triangleq  1 + \alpha \Pcomp + \alpha \Delta_E/u$ and
  $q(u) \triangleq  \frac{\alpha \Delta_E}{u \psi(u)}$.
Then $\psi'(u) = - \alpha \Delta_E/u^2$, and
\begin{equation}
    \begin{split}
q'(u)
&= \alpha \Delta_E \cdot \frac{d}{du}\bigl(u^{-1} \psi(u)^{-1}\bigr) \\[-2pt]
&= \alpha \Delta_E \left[ -u^{-2} \psi^{-1} + u^{-1}\cdot \bigl(-1\bigr) \psi^{-2} \psi'(u) \right] \\[-2pt]
&= -\,\frac{\alpha \Delta_E}{u^{2} \psi(u)} \;+\; \frac{\alpha \Delta_E}{u}\cdot \frac{\alpha \Delta_E/u^{2}}{\psi(u)^{2}} \\[-2pt]
&= -\,\frac{\alpha \Delta_E}{u^{2} \psi(u)} \;+\; \frac{\alpha^{2}\Delta_E^{2}}{u^{3} \psi(u)^{2}}.
\end{split}
\end{equation}
Plugging back, we have
\begin{equation}
   \begin{split}
f''(u)
&= \frac{B}{\ln 2}
\left[
-\frac{\alpha \Delta_E}{u^{2} \psi(u)} - \Bigl(-\frac{\alpha \Delta_E}{u^{2} \psi(u)} + \frac{\alpha^{2}\Delta_E^{2}}{u^{3} \psi(u)^{2}}\Bigr)
\right] \\
&= \frac{B}{\ln 2}
\left[
-\,\frac{\alpha^{2}\Delta_E^{2}}{u^{3} \psi(u)^{2}}
\right] \\
&= -\frac{B\alpha^{2}\Delta_E^{2}}{\ln 2}\frac{1}{u^{3}\, \bigl(1 + \alpha \Pcomp + \alpha \Delta_E/u\bigr)^{2}}\leq0.
\end{split} 
\end{equation}
Therefore, $f'(u)\ge f'(+\infty)>0$.
Hence, we have $R''(n) \leq 0$ and $R'(n) < 0$, which completes the proof.

\subsection{Proof of Theorem \ref{thm:optimal_n}}\label{app:proof_thm2}

Let $u \triangleq T - t_1 - c_1 n - c_2$, and conversely, \(n=a_1-a_2u\) where $a_1 \triangleq \frac{T - t_1 - c_2}{c_1}$, $a_2 \triangleq \frac{1}{c_1}$. This transforms $U(n)$ to
{\predisplaypenalty=0
\begin{equation}
    \hat{U}(u) = - a_2 c_3 u + a_1 c_3 + c_4 + \lambda f(u),
\end{equation}}
where $f(u)$ is from \eqref{f(u)}. 
Consequently, it suffices to solve $\hat U'(u)=0$ for the
stationary point $\tilde u$, from which $\tilde n$ is recovered.
Using the expression for $f'(u)$ from \eqref{f'(u)}, the derivative $\hat{U}'(u)$ is explicitly given by
\begin{equation} \label{utility function w.r.t. u}
\begin{split}
    \hat{U}'(u) = &- a_2 c_3 + \lambda B \Bigg[
    \log_2 \big( 1 + \alpha \Pcomp + \frac{\alpha \Delta_E}{u} \big)
    \autobreak[-1.5ex]-
    \frac{1}{\ln 2}
    \frac{\alpha \Delta_E}{u(1 + \alpha \Pcomp + \alpha \Delta_E/u)}
    \Bigg].
\end{split}
\end{equation}
Setting \(\hat U'(u)=0\) gives
\begin{equation}
    \ln\!\Bigl( 1 + \alpha \Pcomp + \tfrac{\alpha \Delta_E}{u} \Bigr)
-
\frac{\alpha \Delta_E}{\,u\bigl(1 + \alpha \Pcomp + \alpha \Delta_E/u\bigr)}
=
\frac{a_2c_3}{\lambda B}\ln 2.\label{simplified equ}
\end{equation}
Let $D \triangleq1 + \alpha\,\Pcomp$, $K \triangleq \frac{a_2c_3}{\lambda B}\ln 2$, and $C \triangleq 1 + K$. 
Then
\[\ln\!\Bigl(D + \tfrac{\alpha \Delta_E}{u} \Bigr)
\;-\;
\frac{\alpha \Delta_E}{\,u\bigl(D + \alpha \Delta_E/u\bigr)}
\;=\;
K.
\]
Let
\(x \;\triangleq\; D + \frac{\alpha \Delta_E}{u}\). Thus
\(\frac{\alpha \Delta_E}{u} = x - D\) and \(u = \frac{\alpha \Delta_E}{x-D},
\;
x \neq D\). This further simplifies the equation as
\[
\ln x \;-\; \frac{\alpha \Delta_E}{u\,x} \;=\; K.
\]
Plugging \(\frac{\alpha \Delta_E}{u\,x} \;=\; \frac{x-D}{x}\) into it, we have
\[\ln x \;+\; \frac{D}{x} \;=\; 1 + K\;=\;C.\]
Transform it and we have \(x = e^{\,C - D/x} = e^{C}\,e^{-D/x}\). Then, \(
x\,e^{D/x} = e^{C}.\)
Let $y \triangleq\frac{D}{x}$, $x = \frac{D}{y}$, then \(\frac{D}{y}\,e^{y} = e^{C}\) and thus \((-y)\,e^{-y} = -D\,e^{-C}\). Hence \(-y = W_k\!\bigl(-D e^{-C}\bigr),\) and \(x = \frac{D}{y} = -{D}/{W_k\!\bigl(-D e^{-C}\bigr)},\)
with branch index \(k\) chosen to satisfy domain constraints.
For brevity, we denote \(\zeta_k \equiv W_k\!\bigl(-D e^{-C}\bigr)\), and then
$
u \;=\; \frac{\alpha \Delta_E}{-\frac{D}{\zeta_k}-D}
\;=\; -\,\frac{\alpha \Delta_E}{D}\,\frac{\zeta_k}{\zeta_k + 1}.
$
We further investigate the term $y = \frac{D}{x} = \frac{1}{1+\frac{\alpha \Delta_E}{Du}}$. All constituent parameters are positive, with the exception of $\Delta_E $. If $\Delta_E > 0$, it follows that $y < 1$ (implying $-y > -1$); consequently, the solution corresponds to the principal branch $W_0$. If $\Delta_E < 0$, the branch $W_{-1}$ is selected.
Therefore,
\begin{equation}
    \tilde{u} = 
    \begin{cases}
        -\frac{\alpha \Delta_E}{D}\frac{ W_0\bigl(-D e^{-C}\bigr) }{ W_0\bigl(-D e^{-C}\bigr) + 1 }, & \text{if } \Delta_E > 0, \\[1pt]
        -\frac{\alpha \Delta_E}{D}\frac{ W_{-1}\bigl(-D e^{-C}\bigr) }{ W_{-1}\bigl(-D e^{-C}\bigr) + 1 }, & \text{if } \Delta_E < 0.
    \end{cases}
    \label{tilde u}
\end{equation}
Substituting \eqref{tilde u} into \(\tilde{n}
    = a_1 - a_2 \tilde{u}\) yields the results in \eqref{eq:optimal_n_cont}.
Due to the concavity of the objective function, the optimal integer
solution $n^*$ is guaranteed to lie in the candidate set $\mathcal S$
defined in Theorem~\ref{thm:optimal_n}. This completes the proof.

\bibliographystyle{ieeetr}
\bibliography{EHAI/Ref}

@article{cai20253c,
  title={3C Resources Joint Allocation for Time-Deterministic Remote Sensing Image Backhaul in the Space-Ground Integrated Network},
  author={Cai, Chongxiao and Zhu, Yan and Sheng, Min and Li, Jiandong and Shi, Yan and Zhou, Di and Xie, Ziwen and Zhang, Chen},
  journal={arXiv preprint arXiv:2510.09409},
  year={2025}
}

@article{shannon1948mathematical,
  title={A mathematical theory of communication},
  author={Shannon, Claude E},
  journal={The Bell system technical journal},
  volume={27},
  number={3},
  pages={379--423},
  year={1948},
  publisher={Nokia Bell Labs}
}

@article{chen2024survey,
  title={A survey on resource management in joint communication and computing-embedded SAGIN},
  author={Chen, Qian and Guo, Zheng and Meng, Weixiao and Han, Shuai and Li, Cheng and Quek, Tony QS},
  journal={IEEE Commun. Surveys Tuts.},
  volume={27},
  number={3},
  pages={1911--1954},
  year={2024},
  publisher={IEEE}
}

@article{ji2024dynamic,
  title={Dynamic space-ground integrated mobility management strategy for mega LEO satellite constellations},
  author={Ji, Sijing and Zhou, Di and Sheng, Min and Li, Jiandong and Han, Zhu},
  journal={IEEE Trans. Wireless Commun.},
  volume={23},
  number={9},
  pages={11043--11060},
  year={2024},
  publisher={IEEE}
}

@ARTICLE{wen2026ISCC,
  author={Wen, Dingzhu and Xie, Sijing and Cao, Xiaowen and Cui, Yuanhao and Xu, Jie and Shi, Yuanming and Cui, Shuguang},
  journal={IEEE Trans. Wireless Commun.}, 
  title={Integrated Sensing, Communication, and Computation for Over-the-Air Federated Edge Learning}, 
  year={2026},
  volume={25},
  number={},
  pages={2748-2762},
  doi={10.1109/TWC.2025.3598997}}

@inproceedings{liu2024orbit,
  title={In-Orbit Processing or Not? Sunlight-Aware Task Scheduling for Energy-Efficient Space Edge Computing Networks},
  author={Liu, Weisen and Lai, Zeqi and Wu, Qian and Li, Hewu and Zhang, Qi and Li, Zonglun and Li, Yuanjie and Liu, Jun},
  booktitle={IEEE INFOCOM 2024-IEEE Conference on Computer Communications},
  pages={881--890},
  year={2024},
  organization={IEEE}
}

@misc{starcloud_2025_h100,
  author       = {{Starcloud}},
  title        = {Starcloud-1},
  howpublished = {\url{https://www.starcloud.com/starcloud-1}},
  year         = {2025}
}

@article{ding2021joint,
  title={Joint optimization of transmission and computation resources for satellite and high altitude platform assisted edge computing},
  author={Ding, Changfeng and Wang, Jun-Bo and Zhang, Hua and Lin, Min and Li, Geoffrey Ye},
  journal={IEEE Trans. Wireless Commun.},
  volume={21},
  number={2},
  pages={1362--1377},
  year={2021},
  publisher={IEEE}
}

@article{ouyang2023joint,
  title={Joint in-orbit computation and communication for minimizing download time from LEO satellites},
  author={Ouyang, Qiaolin and Ye, Neng and Gao, Jie and Wang, Aihua and Zhao, Lian},
  journal={IEEE Trans. Mob. Comput.},
  volume={23},
  number={5},
  pages={3950--3963},
  year={2023},
  publisher={IEEE}
}

@article{cao2023computing,
  title={Computing-aware routing for LEO satellite networks: A transmission and computation integration approach},
  author={Cao, Jiaqi and Zhang, Shengli and Chen, Qingxia and Wang, Houtian and Wang, Mingzhe and Liu, Naijin},
  journal={IEEE Trans. Veh. Technol.},
  volume={72},
  number={12},
  pages={16607--16623},
  year={2023},
  publisher={IEEE}
}

@article{lai2024joint,
  title={Joint computation offloading and resource allocation for LEO satellite networks using hierarchical multi-agent reinforcement learning},
  author={Lai, Junyu and Liu, Huashuo and Xu, Guoyao and Jiang, Weiwei and Wang, Xiong and Jiang, Dingde},
  journal={IEEE Trans. Cogn. Commun. Netw.},
  volume={11},
  number={4},
  pages={2554--2567},
  year={2024},
  publisher={IEEE}
}

@article{xiao2026sgicpnom,
  title={SGICPNOM: A Computation Offloading Mechanism for 6G Space-Ground Integrated Computing Power Network},
  author={Xiao, Yi and Xu, Xiaolong},
  journal={Computer Networks},
  pages={112082},
  year={2026},
  publisher={Elsevier}
}

@article{sun2026toward,
  title={Toward Communication-Efficient Space Data Centers: Bottlenecks, Architectures, and New Paradigms},
  author={Sun, Minghao and Chen, Zehui and Hou, Jinbo and Wang, Kezhi and Chu, Xiaoli},
  journal={arXiv preprint arXiv:2605.12681},
  year={2026}
}

@article{li2025age,
  title={Age-Critical Joint Communication and Computation Offloading for Satellite-Integrated Internet},
  author={Li, Kun and Jiao, Jian and Huang, Jianhao and Xu, Zhou and Sun, Qunying and Xu, Xiaofan and Wang, Ye and Zhang, Qinyu},
  journal={IEEE Trans. Cogn. Commun. Netw.},
  volume={12},
  pages={4387--4403},
  year={2025},
  publisher={IEEE}
}

@article{zhang2022exploiting,
  title={Exploiting collaborative computing to improve downlink sum rate in satellite integrated terrestrial networks},
  author={Zhang, Liwen and Liu, Junyu and Sheng, Min and Zhao, Nan and Li, Jiandong},
  journal={IEEE Trans. Veh. Technol.},
  volume={72},
  number={4},
  pages={4670--4682},
  year={2022},
  publisher={IEEE}
}

@article{wang2023energy,
  title={Energy-efficient design of satellite-terrestrial computing in 6G wireless networks},
  author={Wang, Qi and Chen, Xiaoming and Qi, Qiao},
  journal={IEEE Trans. Commun.},
  volume={72},
  number={3},
  pages={1759--1772},
  year={2023},
  publisher={IEEE}
}

@article{he2025joint,
  title={Joint data compression and task scheduling for LEO satellite networks},
  author={He, Lijun and Li, Shiyin and Jia, Ziye and Wang, Juncheng and Han, Zhu},
  journal={IEEE Trans. Veh. Technol.},
  year={2025},
  publisher={IEEE}
}

@inproceedings{gong2022computation,
  title={Computation offloading and energy harvesting schemes for sum rate maximization in space-air-ground networks},
  author={Gong, Yongkang and Yao, Haipeng and Xiong, Zehui and Guo, Song and Yu, F Richard and Niyato, Dusit},
  booktitle={GLOBECOM 2022-2022 IEEE Global Communications Conference},
  pages={3941--3946},
  year={2022},
  organization={IEEE}
}

@inproceedings{samsi2023words,
  title={From words to watts: Benchmarking the energy costs of large language model inference},
  author={Samsi, Siddharth and Zhao, Dan and McDonald, Joseph and Li, Baolin and Michaleas, Adam and Jones, Michael and Bergeron, William and Kepner, Jeremy and Tiwari, Devesh and Gadepally, Vijay},
  booktitle={2023 IEEE high performance extreme computing conference (HPEC)},
  pages={1--9},
  year={2023},
  organization={IEEE}
}

@article{qu2026trimcaching,
  title={TrimCaching: Parameter-sharing edge caching for AI model downloading},
  author={Qu, Guanqiao and Lin, Zheng and Chen, Qian and Li, Jian and Liu, Fangming and Chen, Xianhao and Huang, Kaibin},
  journal={IEEE Transactions on Networking},
  year={2026},
  publisher={IEEE}
}

@article{qu2025mobile,
  title={Mobile edge intelligence for large language models: A contemporary survey},
  author={Qu, Guanqiao and Chen, Qiyuan and Wei, Wei and Lin, Zheng and Chen, Xianhao and Huang, Kaibin},
  journal={IEEE Commun. Surveys Tuts.},
  volume={27},
  number={6},
  pages={3820--3860},
  year={2025},
  publisher={IEEE}
}

@article{wei2022chain,
  title={Chain-of-thought prompting elicits reasoning in large language models},
  author={Wei, Jason and Wang, Xuezhi and Schuurmans, Dale and Bosma, Maarten and Xia, Fei and Chi, Ed and Le, Quoc V and Zhou, Denny and others},
  journal={Adv. Neural Inf. Process. Syst.},
  volume={35},
  pages={24824--24837},
  year={2022}
}

@article{song2020denoising,
  title={Denoising diffusion implicit models},
  author={Song, Jiaming and Meng, Chenlin and Ermon, Stefano},
  journal={arXiv preprint arXiv:2010.02502},
  year={2020}
}

@book{wertz2012spacecraft,
  title={Spacecraft attitude determination and control},
  author={Wertz, James R},
  year={2012},
  publisher={Springer Science \& Business Media}
}

@book{bate2020fundamentals,
  title={Fundamentals of astrodynamics},
  author={Bate, Roger R and Mueller, Donald D and White, Jerry E and Saylor, William W},
  year={2020},
  publisher={Courier Dover Publications}
}

@article{qian2025energy,
  title={Energy-efficient data gathering and computing in leo satellite-assisted marine iot networks},
  author={Qian, Li Ping and Fan, Xiaohui and Li, Mingqing and Wu, Yuan},
  journal={IEEE Trans. Cogn. Commun. Netw.},
  year={2025},
  publisher={IEEE}
}

@article{ho2020denoising,
  title={Denoising diffusion probabilistic models},
  author={Ho, Jonathan and Jain, Ajay and Abbeel, Pieter},
  journal={Adv. Neural Inf. Process. Syst.},
  volume={33},
  pages={6840--6851},
  year={2020}
}

@article{dai2025energy,
  title={Energy-efficient multi-access edge computing for heterogeneous satellite-maritime networks: A hybrid harvesting-and-offloading design},
  author={Dai, Minghui and Chang, Shan and Wang, Yixuan and Su, Zhou},
  journal={IEEE Trans. Mob. Comput.},
  year={2025},
  publisher={IEEE}
}

@inproceedings{varan2014energy,
  title={Energy harvesting communications with continuous energy arrivals},
  author={Varan, Burak and Tutuncuoglu, Kaya and Yener, Aylin},
  booktitle={2014 Information Theory and Applications Workshop (ITA)},
  pages={1--10},
  year={2014},
  organization={IEEE}
}

@article{ku2015advances,
  title={Advances in energy harvesting communications: Past, present, and future challenges},
  author={Ku, Meng-Lin and Li, Wei and Chen, Yan and Liu, KJ Ray},
  journal={IEEE Commun. Surveys Tuts.},
  volume={18},
  number={2},
  pages={1384--1412},
  year={2015},
  publisher={IEEE}
}

@article{corless1996lambert,
  title={On the Lambert W function},
  author={Corless, Robert M and Gonnet, Gaston H and Hare, David EG and Jeffrey, David J and Knuth, Donald E},
  journal={Advances in Computational mathematics},
  volume={5},
  number={1},
  pages={329--359},
  year={1996},
  publisher={Springer}
}

@inproceedings{cai2024condition,
  title={Condition-aware neural network for controlled image generation},
  author={Cai, Han and Li, Muyang and Zhang, Qinsheng and Liu, Ming-Yu and Han, Song},
  booktitle={Proceedings of the IEEE/CVF Conference on Computer Vision and Pattern Recognition},
  pages={7194--7203},
  year={2024}
}

@inproceedings{radford2021learning,
  title={Learning transferable visual models from natural language supervision},
  author={Radford, Alec and Kim, Jong Wook and Hallacy, Chris and Ramesh, Aditya and Goh, Gabriel and Agarwal, Sandhini and Sastry, Girish and Askell, Amanda and Mishkin, Pamela and Clark, Jack and others},
  booktitle={International conference on machine learning},
  pages={8748--8763},
  year={2021},
  organization={PmLR}
}

@inproceedings{rombach2022high,
  title={High-resolution image synthesis with latent diffusion models},
  author={Rombach, Robin and Blattmann, Andreas and Lorenz, Dominik and Esser, Patrick and Ommer, Bj{\"o}rn},
  booktitle={Proceedings of the IEEE/CVF conference on computer vision and pattern recognition},
  pages={10684--10695},
  year={2022}
}

@article{yang2016towards,
  title={Towards energy-efficient routing in satellite networks},
  author={Yang, Yuan and Xu, Mingwei and Wang, Dan and Wang, Yu},
  journal={IEEE J. Sel. Areas Commun.},
  volume={34},
  number={12},
  pages={3869--3886},
  year={2016},
  publisher={IEEE}
}

@article{luo2013optimal,
  title={Optimal save-then-transmit protocol for energy harvesting wireless transmitters},
  author={Luo, Shixin and Zhang, Rui and Lim, Teng Joon},
  journal={IEEE Trans. Wireless Commun.},
  volume={12},
  number={3},
  pages={1196--1207},
  year={2013},
  publisher={IEEE}
}

@article{wang2026spacemoe,
  title={SpaceMoE: Realizing Distributed Mixture-of-Experts Inference over Space Networks},
  author={Wang, Zhanwei and Yang, Huiling and Sheng, Min and Letaief, Khaled B and Huang, Kaibin},
  journal={arXiv preprint arXiv:2605.00515},
  year={2026}
}

@article{wang2026revisiting,
  title={Revisiting outage for edge inference systems},
  author={Wang, Zhanwei and Zeng, Qunsong and Zheng, Haotian and Huang, Kaibin},
  journal={IEEE Trans. Commun.},
  year={2026},
  publisher={IEEE}
}

@article{li2024battery,
  title={Battery-aware energy optimization for satellite edge computing},
  author={Li, Qing and Wang, Shangguang and Ma, Xiao and Zhou, Ao and Wang, Yue and Huang, Gang and Liu, Xuanzhe},
  journal={IEEE Transactions on Services Computing},
  volume={17},
  number={2},
  pages={437--451},
  year={2024},
  publisher={IEEE}
}

@article{tutuncuoglu2012optimum,
  title={Optimum transmission policies for battery limited energy harvesting nodes},
  author={Tutuncuoglu, Kaya and Yener, Aylin},
  journal={IEEE Trans. Wireless Commun.},
  volume={11},
  number={3},
  pages={1180--1189},
  year={2012},
  publisher={IEEE}
}

@article{yang2011optimal,
  title={Optimal packet scheduling in an energy harvesting communication system},
  author={Yang, Jing and Ulukus, Sennur},
  journal={IEEE Trans. Commun.},
  volume={60},
  number={1},
  pages={220--230},
  year={2011},
  publisher={IEEE}
}

@article{chen2024fedmeld,
  title={FedMeld: A model-dispersal federated learning framework for space-ground integrated networks},
  author={Chen, Qian and Chen, Xianhao and Huang, Kaibin},
  journal={arXiv preprint arXiv:2412.17231},
  year={2024}
}

@article{esmat2023toward,
  title={Toward resilient network slicing for satellite--terrestrial edge computing IoT},
  author={Esmat, Haitham H and Lorenzo, Beatriz and Shi, Weisong},
  journal={IEEE Internet Things J.},
  volume={10},
  number={16},
  pages={14621--14645},
  year={2023},
  publisher={IEEE}
}

@article{ozel2011transmission,
  title={Transmission with energy harvesting nodes in fading wireless channels: Optimal policies},
  author={Ozel, Omur and Tutuncuoglu, Kaya and Yang, Jing and Ulukus, Sennur and Yener, Aylin},
  journal={IEEE J. Sel. Areas Commun.},
  volume={29},
  number={8},
  pages={1732--1743},
  year={2011},
  publisher={IEEE}
}

@article{cheriet2021inertia,
  title={Inertia tensor estimation for a rigid nadir pointing satellite based on star tracker},
  author={Cheriet, ME and Bellar, Abdellatif and Ghaffour, Mohammed Y and Adnane, Akram and Mohammed, Mohammed A},
  journal={Advances in aircraft and spacecraft science},
  volume={8},
  number={2},
  pages={111--126},
  year={2021},
  publisher={Techno-Press}
}

@article{chen2025space,
  title={Space--ground fluid AI for 6G edge intelligence},
  author={Chen, Qian and Wang, Zhanwei and Chen, Xianhao and Wen, Juan and Zhou, Di and Ji, Sijing and Sheng, Min and Huang, Kaibin},
  journal={Engineering},
  year={2025},
  publisher={Elsevier}
}

\end{document}